\documentclass[11pt]{amsart}
\def\isdraft{1}

\usepackage[dvipsnames]{xcolor}
\usepackage{bbm}
\usepackage{graphicx}
\usepackage{amsaddr} 
\usepackage{mathtools}
\usepackage{amsmath,amssymb,amsfonts,dsfont}
\usepackage{prettyref} 
\usepackage[utf8]{inputenc}
\usepackage{enumitem}
\usepackage[hyphens]{url} 
\usepackage{tikz-cd}
\usepackage{tikz}
\usetikzlibrary{positioning}
\usetikzlibrary{arrows}
\usepackage{bussproofs}
\usepackage[mathscr]{euscript}
\usepackage{hyperref}
\usepackage{amsthm}
\usepackage{theapa}
\usepackage{a4wide}
\usepackage[color=black,textcolor=white\if\isdraft0,disable\fi]{todonotes}
\usepackage{stackengine}
\usepackage{stmaryrd}
\usepackage{wrapfig}
\usepackage{pdfpages}
\usepackage{marginnote}
\usepackage{float}
\graphicspath{{fig/}}
\usetikzlibrary{arrows,automata,topaths,matrix,positioning,fit}
\usepgflibrary{shapes.geometric}
\usetikzlibrary{shapes.geometric}
\tikzset{every state/.style={minimum size=0pt}}
\usepackage{footnote}
\usepackage{float}
\usepackage{tabularx}
\usepackage{txfonts}

\newtheorem{theorem}{Theorem}

\newtheorem{corollary}[theorem]{Corollary}
\newtheorem{fact}[theorem]{Fact}
\newtheorem{lemma}[theorem]{Lemma}
\newtheorem{proposition}[theorem]{Proposition}

\theoremstyle{definition} 

\newtheorem{definition}[theorem]{Definition}

\newtheorem{example}[theorem]{Example}

\newtheorem{notation}[theorem]{Notation}

\newtheorem{remark}[theorem]{Remark}

\newrefformat{§}{§\ref{#1}}
\newrefformat{algorithm}{Algorithm \ref{#1}}
\newrefformat{a}{Answer \ref{#1}}
\newrefformat{appendix}{Appendix \ref{#1}}
\newrefformat{claim}{Claim \ref{#1}}
\newrefformat{conclusion}{Conclusion \ref{#1}}
\newrefformat{convention}{Convention \ref{#1}}
\newrefformat{cj}{Conjecture \ref{#1}}
\newrefformat{c}{Corollary \ref{#1}}
\newrefformat{equ}{(\ref{#1})}
\newrefformat{e}{Example \ref{#1}}
\newrefformat{exe}{Exercise \ref{#1}}
\newrefformat{d}{Definition \ref{#1}}
\newrefformat{done}{Done \ref{#1}}
\newrefformat{f}{Fact \ref{#1}}
\newrefformat{fig}{Figure \ref{#1}}
\newrefformat{h}{Hypothesis \ref{#1}}
\newrefformat{idea}{Idea \ref{#1}}
\newrefformat{i}{Item \ref{#1}}
\newrefformat{l}{Lemma \ref{#1}}
\newrefformat{note}{Note \ref{#1}}
\newrefformat{n}{Notation \ref{#1}}
\newrefformat{o}{Observation \ref{#1}}
\newrefformat{problem}{Problem \ref{#1}}
\newrefformat{p}{Proposition \ref{#1}}
\newrefformat{pseudocode}{Pseudocode \ref{#1}}
\newrefformat{Q}{Q. \ref{#1}}
\newrefformat{r}{Remark \ref{#1}}
\newrefformat{table}{Table \ref{#1}}
\newrefformat{t}{Theorem \ref{#1}}
\newrefformat{Todo}{Todo \ref{#1}}
\newrefformat{w}{Warning \ref{#1}}

\title{
	Sequential composition of answer set programs
}
\author{
    Christian Anti\'c
}
\address{
    christian.antic@icloud.com\\
    Vienna University of Technology\\
    Vienna, Austria
}

\begin{document}

\maketitle

\begin{abstract} 
    This paper contributes to the mathematical foundations of logic programming by introducing and studying the sequential composition of answer set programs. On the semantic side, we show that the immediate consequence operator of a program can be represented via composition, which allows us to compute the least model semantics of Horn programs without any explicit reference to operators. As a result, we can characterize answer sets algebraically, which further provides an algebraic characterization of strong and uniform equivalence which is appealing. This bridges the conceptual gap between the syntax and semantics of an answer set program in a mathematically satisfactory way. The so-obtained algebraization of answer set programming allows us to transfer algebraic concepts into the ASP-setting which we demonstrate by introducing the index and period of an answer set program as an algebraic measure of its cyclicality. The technical part of the paper ends with a brief section introducing the algebraically inspired novel class of aperiodic answer set programs strictly containing the acyclic ones. In a broader sense, this paper is a first step towards an algebra of answer set programs and in the future we plan to lift the methods of this paper to wider classes of programs, most importantly to higher-order and disjunctive programs and extensions thereof.
\end{abstract}

\section{Introduction}

Non-monotonic reasoning is an essential part of human intelligence prominently formalized in artificial intelligence research via answer set programming \cite<see e.g.>{Brewka11,Lifschitz19,Baral03} among other formalizations \cite<see e.g.>{Antoniou97}. Answer set programs are rule-based systems with the rules and facts being written in a sublanguage of predicate or propositional logic extended by a unary non-monotonic operator ``$not$'' denoting \textit{negation as failure} (or \textit{default negation}) \cite{Clark78}. While each monotone (i.e., negation-free) answer set program has a unique least Herbrand model (with the least model semantics \cite{vanEmden76} being the accepted semantics for this class of programs), for general answer set programs a large number of different purely declarative semantics exist. Many of it have been introduced some 30 years ago, among them the \textit{answer set semantics} \cite{Gelfond91}. With the emergence of efficient solvers such as DLV \cite{Leone06}, Smodels \cite{Simons02}, Cmodels \cite{Giunchiglia06}, and Clasp \cite{Gebser12}, programming under answer set semantics led to a predominant declarative problem solving paradigm, called \textit{answer set programming} (or \textit{ASP}) \cite{Marek99,Lifschitz02}. Answer set programming has a wide range of applications and has been successfully applied to various AI-related subfields such as planning and diagnosis (for a survey see \citeA{Brewka11,Eiter09,Baral03}). Driven by this practical needs, a large number of extensions of classical answer set programs have been proposed, e.g. \textit{aggregates} \cite<cf.>{Faber04,Faber11,Pelov04}, \textit{choice rules} \cite{Niemela99}, \textit{dl-atoms} \cite{Eiter08a}, and general \textit{external atoms} \cite{Eiter05}. For excellent introductions to the field of answer set programming we refer the reader to \citeA{Brewka11,Baral03,Eiter09}.

The \textbf{algebraization} of parts of theoretical computer science has a long history beginning with the algebraic theory of finite automata already well established in the late 1960s \cite<see e.g.>{Ginzburg68,Arbib68}. The \textbf{algebraization of logic programming} started in the mid 1980s with \citeS{OKeefe85} work towards an algebra for constructing logic programs, implicitly continued by work on \textbf{modular logic programming} in the 1980s and 1990s \shortcite<see e.g.>{Bossi96,Brogi92,Brogi92a,Brogi95,Brogi99,Bugliesi94,Dong90,Gaifman89,Hill94,Mancarella88,OKeefe85}, and more recently on modular answer set programming \shortcite<see e.g.>{DaoTran09,Oikarinen06,Oikarinen06a}. 

The \textbf{purpose of this paper} is to add to the logic programmer's repertoire another \textbf{modularity operation} in the form of the \textit{\textbf{sequential composition}} of answer set programs which is naturally induced by their rule-like structure, and to initiate the \textbf{algebraization of answer set programming}. The main \textbf{motivation} of this work is not only to obtain an algebraic description of the answer set semantics (\prettyref{t:T_P}) --- thus adding one more characterization to the thirteen listed by \citeA{Lifschitz10} --- and notions of equivalence (Corollaries \ref{c:strong}, \ref{c:uniform}), but also to pave the way for algebraic notions to be introduced into the ASP-setting as explained below.

For this, we lift the recently introduced concepts and results in \citeA{Antic21-1} from propositional Horn logic programs to answer set programs containing negation as failure. This task turns out to be non-trivial due to the intricate algebraic properties of composing negation as failure occurring in rule bodies. The rule-like structure of answer set programs naturally induces the compositional structure of a unital magma (a set closed under a given binary operation) on the space of all answer set programs as we will demonstrate in this paper. Specifically, we show that the notion of composition gives rise to a family of finite magmas and algebras of answer set programs (\prettyref{t:ASP}), which we will call \textit{\textbf{ASP magma}} and \textit{\textbf{ASP algebra}} in this paper. We also show that the restricted class of proper Krom-Horn programs, which only contain rules with exactly one body atom, yields an idempotent semiring (\prettyref{t:Krom}).\footnote{Krom-Horn programs are called ``unary'' in \citeA{Janhunen06}.} On the semantic side, we show that the van Emden-Kowalski immediate consequence operator of a program can be represented via composition (Theorem \ref{t:T_P}), which allows us to compute the answer set semantics of programs without any explicit reference to operators (Theorem \ref{t:AS}). This bridges the conceptual gap between the syntax and semantics of an answer set program in a mathematically satisfactory way. Moreover, it provides an algebraic characterization of strong (\prettyref{c:strong}) and uniform equivalence (\prettyref{c:uniform}). 

Due to their global nature, answer sets tend to have limitations when it comes to modular programming. The rough observation is that the defining rules of each atom must be treated together and the same applies to positively interdependent atoms. This means that in general, the answer sets of a program cannot be computed from the answer sets of its factors via composition. We are thus more interested in the algebraic structure of the \textit{syntax} of programs. It can thus be said that the paper is part of \textbf{algebraic answer set programming} and not so much of modular answer set programming. That is, our inspirations come from algebraically syntactic and not so much from semantic considerations. For example, in \prettyref{§:Index} we will introduce the \textit{\textbf{index}} and \textit{\textbf{period}} of an answer set program in analogy to the same concepts in semigroup theory \cite<see>[p.10]{Howie03}. Numerous examples demonstrate that these notions provide a suitable algebraization of the cyclicality of a program. We then end the technical part of the paper by introducing the purely algebraically inspired class of \textit{\textbf{aperiodic answer set programs}} strictly containing the acyclic ones \cite{Apt91}, which hopefully convinces the reader of the benefit of having an algebraic theory of answer set programs. In fact, given the counterexample to an acyclic program in \prettyref{e:not_aperiodic}, it seems to us that the notion of an aperiodic program better captures the intuition of an ``acyclic'' program than the usual definition in terms of a level mapping. This is an instance of a more general future research program of interpreting algebraic concepts in the ASP-setting. For example, in \citeA[§Related work]{Antic21-1}, it has been observed that from a purely mathematical point of view, sequential composition of propositional Horn logic programs is very similar to composition in multicategories \cite<cf.>[§ 2.1]{Leinster04} and it remains to be seen whether interesting connections between those two seemingly distant areas emerge.

From an artificial intelligence perspective, we obtain a novel and useful algebraic operation for the composition and decomposition of answer set programs, which in combination with the abstract algebraic framework of analogical proportions in \citeA{Antic22} can be used for answer set program synthesis similar to \citeA{Antic23-23}, which remains a promising line of future work (see \prettyref{§:FW}).

In a broader sense, this paper is a further step towards an algebra of rule-based logical theories and in the future we plan to adapt and generalize the methods of this paper to wider classes of programs, most importantly for higher-order logic programs \cite{Chen93,Miller12} and disjunctive answer set programs \cite{Eiter97} and extensions thereof.

\section{Preliminaries}

This section recalls the syntax and semantics of answer set programming, and some algebraic structures occurring in the rest of the paper.

\subsection{Algebraic structures}

We define the composition $f\circ g$ of two functions $f$ and $g$ by $(f\circ g)(x):=f(g(x))$. Given two sets $A$ and $B$, we will write $A\subseteq_k B$ in case $A$ is a subset of $B$ with $k$ elements, for some non-negative integer $k$. 
We denote the \textit{\textbf{identity function}} on a set $A$ by $Id_A$.

A \textit{\textbf{partially ordered set}} (or \textit{\textbf{poset}}) is a set $L$ together with a reflexive, transitive, and anti-symmetric binary relation $\leq$ on $L$. A \textit{\textbf{prefixed point}} of an operator $f$ on a poset $L$ is any element $x\in L$ such that $f(x)\leq x$; moreover, we call any $x\in L$ a \textit{\textbf{fixed point}} of $f$ if $f(x)=x$. A \textit{\textbf{lattice}} is a poset where any two elements $x,y\in L$ have a unique supremum and a unique infimum in $L$, and a lattice is \textit{\textbf{complete}} iff any subset of $L$ has an upper and a lower bound in $L$.

A \textit{\textbf{magma}} is a set $M$ together with a binary operation $\cdot$ on $M$. We call $(M,\cdot,1)$ a \textit{\textbf{unital magma}} if it contains a unit element 1 such that $1x = x1 = x$ holds for all $x\in M$. A \textit{\textbf{semigroup}} is a magma $(S,\cdot)$ in which $\cdot$ is associative. A \textit{\textbf{monoid}} is a semigroup containing a unit element 1 such that $1x=x1=x$ holds for all $x$. A \textit{\textbf{group}} is a monoid which contains an inverse $x^{-1}$ for every $x$ such that $xx^{-1}=x^{-1}x=1$. A \textit{\textbf{left}} (resp., \textit{\textbf{right}}) \textit{\textbf{zero}} is an element $0$ such that $0x=0$ (resp., $x0=0$) holds for all $x\in S$. An \textit{\textbf{ordered semigroup}} is a semigroup $S$ together with a partial order $\leq$ that is compatible with the semigroup operation, meaning that $x\leq y$ implies $zx\leq zy$ and $xz\leq yz$ for all $x,y,z\in S$. An \textit{\textbf{ordered monoid}} is defined in the obvious way. 
A non-empty subset $I$ of $S$ is called a \textit{\textbf{left}} (resp., \textit{\textbf{right}}) \textit{\textbf{ideal}} if $SI\subseteq I$ (resp., $IS\subseteq I$), and a (\textit{\textbf{two-sided}}) \textit{\textbf{ideal}} if it is both a left and right ideal.  

A \textit{\textbf{seminearring}} is a set $S$ together with two binary operations $+$ and $\cdot$ on $S$, and a constant $0\in S$, such that $(S,+,0)$ is a monoid and $(S,\cdot)$ is a semigroup satisfying the following laws:
\begin{enumerate}
    \item $(x+y)\cdot z=x\cdot z+y\cdot z$ for all $x,y,z\in S$ (right-distributivity); and
    \item $0\cdot x=0$ for all $x\in S$.
    \end{enumerate} The seminearring $S$ is called \textit{\textbf{idempotent}} if $x+x=x$ holds for all $x\in S$. A \textit{\textbf{semiring}} is a seminearring $(S,+,\cdot,0,1)$ such that $(S,\cdot,1)$ is a monoid, $+$ is commutative, and additionally to the laws of a seminearring the following laws are satisfied:
    \begin{enumerate}
    \item $x\cdot (y+z)=x\cdot y+x\cdot z$ for all $x,y,z\in S$ (left-distributivity); and
    \item $x\cdot 0=0$ for all $x\in S$.
\end{enumerate}

\subsection{Answer set programs}

We recall the syntax and semantics of answer set programming by mainly following the lines of \citeA{Baral03}.

\subsubsection{Syntax}

In the rest of the paper, $A$ denotes a finite alphabet of propositional atoms not containing the special symbols $\mathbf t$ (true) and $\mathbf f$ (false). A \textit{\textbf{literal}} is either an atom $a$ or a negated atom $not\ a$, where ``$not$'' denotes \textit{\textbf{negation as failure}} \cite{Clark78}.

An (\textit{\textbf{answer set}}) \textit{\textbf{program}} over $A$ is a finite set of \textit{\textbf{rules}} of the form
\begin{align}\label{equ:r} 
    a_0\leftarrow a_1,\ldots,a_\ell,not\ a_{\ell+1},\ldots,not\ a_k,\quad k\geq\ell\geq 0,
\end{align} where $a_0,\ldots,a_k\in A$ are atoms, and we denote the set of all answer set programs over $A$ by $\mathbb P_A$ or simply by $\mathbb P$ in case $A$ is understood. It will be convenient to define, for a rule $r$ of the form \prettyref{equ:r},
\begin{align*} 
    head(r) &:= \{a_0\},\\
    body(r) &:= \{a_1,\ldots,a_\ell,not\ a_{\ell+1},\ldots,not\ a_k\},\\
\end{align*} extended to sequences of programs, for $n\geq 1$, by 
\begin{align*} 
    head(P_1,\ldots,P_n) &:= \bigcup_{r\in P_1\cup\ldots\cup P_n}head(r)\\
    body(P_1,\ldots,P_n) &:= \bigcup_{r\in P_1\cup\ldots\cup P_n}body(r).
\end{align*} Moreover, we define
\begin{align*} 
    body_+(r) := \{a_1,\ldots,a_\ell\} \quad\text{and}\quad body_-(r) := \{a_{\ell+1},\ldots,a_k\},
\end{align*} extended to sequences of rules, for $n\geq 1$, by
\begin{align*} 
    body_+(r_1,\ldots,r_n) := \bigcup_{i=1}^n body_+(r_i) \quad\text{and}\quad body_-(r_1,\ldots,r_n) := \bigcup_{i=1}^n body_-(r_i).
\end{align*}

The \textit{\textbf{size}} of a rule $r$ of the form \prettyref{equ:r} is $k$ and denoted by $sz(r)$. A rule $r$ of the form \prettyref{equ:r} is \textit{\textbf{positive}} or \textit{\textbf{Horn}} if $\ell=k$, and \textit{\textbf{negative}} if $\ell=0$. A program is \textit{\textbf{positive}} or \textit{\textbf{Horn}} (resp., \textit{\textbf{negative}}) if it contains only positive (resp., negative) rules. 
Facts are the only kind of rules which are positive and negative at the same time. Define the positive and negative part of $r$ respectively by
\begin{align*} 
    r_+ := a_0\leftarrow a_1,\ldots,a_\ell \quad\text{and}\quad r_- := a_0\leftarrow not\ a_{\ell+1},\ldots,not\ a_k.
\end{align*} extended to programs rule-wise by
\begin{align*} 
    P_+ := \{r_+\mid r\in P\} \quad\text{and}\quad P_- := \{r_-\mid r\in P\}.
\end{align*}

Moreover, define the \textit{\textbf{hornification}} of $r$ by $$horn(r):=a_0\leftarrow a_1,\ldots,a_k,$$ extended to programs rule-wise via
\begin{align*} 
    horn(P) := \{horn(r)\mid r\in P\}.
\end{align*}

A \textit{\textbf{fact}} is a rule with empty body and a \textit{\textbf{proper}} rule is a rule which is not a fact. The facts and proper rules of a program $P$ are denoted by $facts(P)$ and $proper(P)$, respectively. 

A Horn rule $r$ is called \textit{\textbf{Krom}}\footnote{Krom-Horn rules where first introduced and studied by \citeA{Krom67}; they are known in ASP as \textit{\textbf{unary}} programs \cite{Janhunen06}.} if it has at most one body atom. 
A Horn program is \textit{\textbf{Krom}} if it contains only Krom rules. 

We call a program \textit{\textbf{minimalist}} if it contains at most one rule for each rule head.

A Horn program $H$ is \textit{\textbf{acyclic}} \cite{Apt91} if there is a mapping $\ell: A\to \mathbb N$ such that for each rule $r\in P$, we have $\ell(head(r))> \ell(body(r))$, and in this case we call $\ell$ a \textit{\textbf{level mapping}} for $H$. 

\vspace*{0.5cm}

An \textit{\textbf{$\lor$-program}} over $A$ is a finite set of \textit{\textbf{$\lor$-rules}} of the form
\begin{align*} 
    a\leftarrow B_1\lor\ldots\lor B_m,\quad m\geq 0,
\end{align*} where $a\in A$ is an atom and $B_1,\ldots,B_m$ are sets of literals. $\lor$-Programs containing disjunction in rule bodies will only occur as intermediate steps in the computation of composition below.

Define the \textit{\textbf{dual}} of a Horn program $H$ by\footnote{This is not to be confused with the dual programs in \citeA{Fichte15}.}
\begin{align*} 
    H^d := facts(H)\cup\{a\leftarrow head(r)\mid r\in proper(P), a\in body(r)\}.
\end{align*} Roughly, we obtain the dual of a Horn program by reversing all the arrows of its proper rules: It consists of the facts in $H$ together with $a_i \leftarrow a_0$, $1\leq i\leq k$, for each proper rule $a_0 \leftarrow a_1,\ldots,a_k\in H$; that is, every proper rule of the form $a_0 \leftarrow a_1, \ldots, a_k$ amounts to the rules $a_1 \leftarrow a_0, \ldots, \quad\ldots\quad a_k \leftarrow a_0$.

\subsubsection{Semantics}\label{§:Semantics}

An \textit{\textbf{interpretation}} is any subset of $A$. The \textit{\textbf{entailment relation}} with respect to an interpretation $I$ and a program $P$ is defined inductively as follows:
\begin{enumerate}
    \item for an atom $a\in A$, $I\models a$ if $a\in I$, and $I\models not\ a$ if $a\not\in I$;

    \item for a set of literals $L$, $I\models L$ if $I\models l$ for every $l\in L$;

    \item for a rule $r$ of the form \prettyref{equ:r}, $I\models r$ if $I\models head(r)$ or $I\not\models body(r)$, which is equivalent to $body_+(r)\not\subseteq I$ or $I\cap body_-(r)\neq\emptyset$;

    \item $I\models P$ if $I\models r$ holds for each rule $r\in P$. In case $I\models P$, we call $I$ a \textit{\textbf{model}} of $P$.
\end{enumerate} It is well-known that the space of all models of a \textit{Horn} program $H$ forms a complete lattice containing a \textit{\textbf{least model}} denoted by $LM(H)$.

Define the \textit{\textbf{Gelfond-Lifschitz reduct}} of $P$ with respect to $I$ by the Horn program
\begin{align*} 
    gP^I := \{r_+\mid r\in P,\, I\cap body_-(r)=\emptyset\}.
\end{align*} Moreover, define the \textit{\textbf{left}} and \textit{\textbf{right reduct}} of $P$, with respect to some interpretation $I$, by
\begin{align*} 
    ^IP := \{r\in P\mid I\models head(r)\}\quad\text{and}\quad P^I:=\{r\in P\mid I\models body(r)\}.
\end{align*} Of course, our notion of right reduct is identical to the \textit{\textbf{Faber-Leone-Pfeifer reduct}} \cite{Faber11} well-known in answer set programming.

An interpretation $I$ is an \textit{\textbf{answer set}} of $P$ if $I$ is the least model of $gP^I$.

Define the \textit{\textbf{van Emden-Kowalski operator}} of $P$, given an interpretation $I$, by
\begin{align*} 
    T_P(I) := \{head(r)\mid r\in P:I\models body(r)\}.
\end{align*} It is well-known that the least model semantics of a Horn program coincides with the least fixed point of its associated van Emden-Kowalski operator \cite{vanEmden76}. We call an interpretation $I$ a \textit{\textbf{supported model}} of $P$ if $I$ is a fixed point of $T_P$.

The following results are answer set programming folklore \cite<see e.g.>{Apt90,Lloyd87}:

\begin{theorem}\label{t:semantics} Let $H$ be a Horn program and let $P$ be an answer set program.
\begin{enumerate}
    \item An interpretation $I$ is a model of $P$ iff $I$ is a prefixed point of $T_P$.
    \item The least model of $H$ coincides with the least fixed point of $T_H$.
    \item An interpretation $I$ is an answer set of $P$ iff $I$ is the least fixed point of $T_{gP^I}$. 
    \item An interpretation $I$ is an answer set of $P$ iff $I$ is a subset minimal model of $P^I$.
\end{enumerate}
\end{theorem}

We say that $P$ and $R$ are (i) \textit{\textbf{equivalent}} if $P$ and $R$ have the same answer sets; (ii) \textit{\textbf{subsumption equivalent}} if $T_P = T_R$ \cite{Maher88}; (iii) \textit{\textbf{uniformly equivalent}} if $P\cup I$ is equivalent to $R\cup I$ for any interpretation $I$ \cite{Eiter03}; and (iv) \textit{\textbf{strongly equivalent}} if $P\cup Q$ is equivalent to $R\cup Q$ for any program $Q$ \cite{Lifschitz01}.

\section{Sequential composition}\label{§:C}

This is the main section of the paper. Here we define the sequential composition of answer set programs and prove the main theorem of this work --- \prettyref{t:ASP} --- that it induces the structure of a unital magma on the space of all answer set programs.

Before we give the formal definition of composition below, we shall first introduce some auxiliary constructions. The goal is to extend the ``$not$'' operator from atoms to programs, which will be essential in the definition of composition below. 

\begin{notation} In the rest of the paper, $P$ and $R$ denote answer set programs, and $I$ denotes an interpretation over some joint finite alphabet of propositional atoms $A$.
\end{notation}

First, define the \textit{\textbf{$\mathbf{tf}$-operator}} by
\begin{align*} 
    \mathbf{tf}(P) := proper(P)\cup\{a\leftarrow\mathbf t\mid a\in P\}\cup \{a \leftarrow \mathbf f\mid a\in A-head(P)\}.
\end{align*} Roughly, the $\mathbf{tf}$-operator replaces every fact $a$ in $P$ by $a\leftarrow\mathbf t$, and it makes every atom $a$ not occurring in any rule head of $P$ explicit by adding $a\leftarrow\mathbf f$. This transformation will be needed in the treatment of negation in the definition of composition. The reader should interpret $a\leftarrow\mathbf t$ as ``$a$ is true'' and $a\leftarrow\mathbf f$ as ``$a$ is false'', similar to truth value assignments in imperative programming. Notice that the $\mathbf{tf}$-operator depends implicitly on the underlying alphabet $A$.

Second, define the \textit{\textbf{overline operator}} by
\begin{align*} 
    \overline{P} := \{head(r)\leftarrow (body(r)-\{\mathbf t\})\mid r\in P\} - \{r\mid\mathbf f\in body(r)\}.
\end{align*} Roughly, the overline operator removes every occurrence of $\mathbf t$ from rule bodies and eliminates every rule containing $\mathbf f$ and is therefore ``dual'' to the $\mathbf{tf}$-operator.

Third, define the \textit{\textbf{$\lor$-operator}} by\footnote{Recall from \prettyref{§:Semantics} that $^{head(r)}P$ is the left reduct of $P$ with respect to $head(r)$.}
\begin{align*} 
    P^\lor := \left\{head(r) \leftarrow \bigvee body\left({^{head(r)}}P\right)\;\middle|\; r\in P\right\}.
\end{align*} Intuitively, the $\lor$-program $P^\lor$ contains exactly one $\lor$-rule for each head atom in $P$ containing the disjunction of all rule bodies with the same head atom. Notice the similarity to the well-known completion of a program \cite{Clark78}.

Fourth, define the negation of an $\lor$-program inductively as follows. First, define
\begin{align*} 
    not\ \mathbf t := \mathbf f \quad\text{and}\quad not\ \mathbf f:=\mathbf t,
\end{align*} extended to a literal $L$ by
\begin{align*} 
    not\ L := \begin{cases}
        a & \text{if } L=not\ a,\\
        not\ a & \text{if } L=a.
    \end{cases}
\end{align*} Then, for an $\lor$-rule $r$ of the form
\begin{align*} 
    head(r) \leftarrow \{L_1^1,\ldots,L_{k_1}^1\}\lor\ldots\lor\{L_1^{n},\ldots,L_{k_n}^n\},
\end{align*} where each $L_i^j$ is a literal, $1\leq j\leq n$, $n\geq 1$,\footnote{We assume here that the $\lor$-rule $r$ contains at least one body literal possibly consisting only of a truth value among $\mathbf t$ and $\mathbf f$. This is consistent with our use of negation below where each atom $a$ is first translated via the $\mathbf{tf}$-operator into $a\leftarrow\mathbf t$.} $1\leq i\leq k_j$, $k_j\geq 1$, define
\begin{align*} 
    not\ r
        &:= \bigcup_{1\leq i_1\leq k_1}\dots\bigcup_{1\leq i_n\leq k_n}\left\{head(r)\leftarrow not\ L_{i_1}^1,\ldots,not\ L_{i_n}^n\right\}.
\end{align*} Now define the \textit{\textbf{negation}} of a program $P$ rule-wise by
\begin{align*} 
    not\ P := \left\{\overline{not\ r}\;\middle|\; r\in\mathbf{tf}(P)^\lor\right\}.
\end{align*} Notice that the $not$-operator depends implicitly on the underlying alphabet $A$. 

We are now ready to introduce the main notion of the paper:

\begin{definition}\label{d:circ} Define the (\textit{\textbf{sequential}}) \textit{\textbf{composition}} of $P$ and $R$ by
\begin{align}\label{equ:circ} 
    P\circ R := \left\{head(r)\leftarrow body(S,N)\;\middle|\;
    \begin{array}{l}
        r\in P\\
        S\subseteq_{sz(r_+)} R\\
        N\subseteq_{sz(r_-)} not\ R\\
        head(S) = body_+(r)\\
        head(N) = body_-(r)
    \end{array}\right\}.
\end{align} We will often write $PR$ instead of $P\circ R$ in case composition is understood.
\end{definition}

Roughly, the composition of $P$ and $R$ is computed by resolving all body literals in $P$ with ``matching'' rule heads of $R$, where the intermediate programs $S$ and $N$ ``bridge'' the positive and negative part of a rule in $P$ with the bodies in $R$, respectively. 

Before proceeding with formal constructions and results, we first want to illustrate composition with the following simple example:

\begin{example} Consider the rule $r$ and program $R$ given by
\begin{align*} r:=a\leftarrow not\ b \quad\text{and}\quad R:= \left\{
\begin{array}{l}
    b\leftarrow not\ c,not\ d\\
    b\leftarrow c,d
\end{array}
\right\}.
\end{align*} We wish to compute $\{r\}\circ R$. For this, we first compute
\begin{align*} 
    {\bf tf}(R)^\lor = \left\{
    \begin{array}{l}
        a\leftarrow\mathbf f\\
        b\leftarrow\{not\ c,not\ d\}\lor\{c,d\}\\
        c\leftarrow\mathbf f\\
        d\leftarrow\mathbf f
    \end{array}
    \right\}
\end{align*} and then
\begin{align*} 
    not\ R &= \left\{
    \begin{array}{l}
        \overline{a\leftarrow not\ \mathbf f}\\
        \overline{b\leftarrow not(\{not\ c,not\ d\}\lor\{c,d\})}\\
        \overline{c\leftarrow not\ \mathbf f}\\
        \overline{d\leftarrow not\ \mathbf f}
    \end{array}
    \right\}\\
    &= \left\{
    \begin{array}{l}
        a\\
        \overline{b\leftarrow not\{not\ c,not\ d\},not\{c,d\}}\\
        c\\
        d
    \end{array}
    \right\}\\
    &= \left\{
    \begin{array}{l}
        a\\
        \overline{b\leftarrow \{c\lor d\},\{not\ c\lor not\ d\}}\\
        c\\
        d
    \end{array}
    \right\}= \left\{
    \begin{array}{l}
        a\\
        b\leftarrow c,not\ c\\
        b\leftarrow d,not\ c\\
        b\leftarrow c,not\ d\\
        b\leftarrow d,not\ d\\
        c\\
        d
    \end{array}
    \right\}.
\end{align*} We can now compute
\begin{align*} 
    \{r\}\circ R = \left\{
    \begin{array}{l}
        a\leftarrow c,not\ c\\
        a\leftarrow d,not\ c\\
        a\leftarrow c,not\ d\\
        a\leftarrow d,not\ d
    \end{array}
    \right\}.
\end{align*}
\end{example}

Notice that we can reformulate composition as
\begin{align}\label{equ:bigcup} 
    P\circ R = \bigcup_{r\in P}(\{r\}\circ R),
\end{align} which directly implies right-distributivity of composition, that is,
\begin{align}\label{equ:P_cup_R_Q} 
    (P\cup R)\circ Q=(P\circ Q)\cup (R\circ Q)\quad\text{holds for all programs }P,Q,R.
\end{align} However, the following counter-example shows that left-distributivity fails in general:  
\begin{align*} 
    \{a\leftarrow b,c\}\circ(\{b\}\cup\{c\})=\{a\}
\end{align*} whereas
\begin{align*} 
    (\{a\leftarrow b,c\}\circ\{b\})\cup(\{a\leftarrow b,c\}\circ\{c\})=\emptyset.
\end{align*} 
The situation changes for Krom-Horn programs (see \prettyref{t:Krom}).

We can write $P$ as the union of its facts and proper rules, that is,
\begin{align}\label{equ:P} 
    P=facts(P)\cup proper(P).
\end{align} Hence, we can rewrite the composition of $P$ and $R$ as
\begin{align}\label{equ:facts_proper} 
    P\circ R &= (facts(P)\cup proper(P))R \stackrel{\prettyref{equ:P_cup_R_Q}}=facts(P)R\cup proper(P)R = facts(P)\cup proper(P)R,
\end{align} which shows that the facts in $P$ are preserved by composition, that is, we have
\begin{align}\label{equ:facts_PR} 
    facts(P)\subseteq facts(P\circ R).
\end{align}

The following example shows that, unfortunately, composition is \textit{not} associative even in the Horn case (but see \prettyref{t:ASP}).

\begin{example}\label{exa:non-associativity} Consider the Horn rule
\begin{align*} 
    r := a\leftarrow b,c,
\end{align*} and the Horn programs
\begin{align*} 
    P := \left\{
    \begin{array}{l}
        b\leftarrow b\\
        c\leftarrow b,c
    \end{array}
    \right\} \quad\text{and}\quad R:= \left\{
    \begin{array}{l}
        b\leftarrow d\\
        b\leftarrow e\\
        c\leftarrow f
    \end{array}
\right\}.
\end{align*} A simple computation yields
\begin{align*} 
    \{r\}(PR) = \left\{
    \begin{array}{l}
        a\leftarrow d,f\\
        a\leftarrow e,f\\
        a\leftarrow d,e,f\\
    \end{array}
    \right\}\neq \left\{
    \begin{array}{l}
      a\leftarrow d,f\\
      a\leftarrow e,f
    \end{array}
    \right\} = (\{r\}P)R.
\end{align*}
\end{example}

\begin{definition} Define the \textit{\textbf{unit program}} by the Krom-Horn program
\begin{align*} 1_A:=\{a\leftarrow a\mid a\in A\}.
\end{align*} We will often omit the reference to the underlying alphabet $A$.
\end{definition}

We are now ready to prove the main structural result of the paper:\footnote{In the rest of the paper, all statements about spaces of programs are always with respect to some fixed underlying finite alphabet $A$.}

\begin{theorem}\label{t:ASP} The space of all answer set programs over some fixed alphabet forms a finite unital magma with respect to composition ordered by set inclusion with the neutral element given by the unit program. Moreover, the empty program is a left zero and composition distributes from the right over union, that is, for any answer set programs $P,Q,R$ we have
\begin{align} 
    \label{equ:P_cup_R_circ_R} (P\cup R)\circ Q&=(P\circ Q)\cup (R\circ Q).
\end{align}
\end{theorem}
\begin{proof} The composition of two programs is again a program, which is not completely obvious since $\lor$-programs possibly containing $\mathbf t$ and $\mathbf f$ occur in intermediate steps in the computation of composition. The reason is that the $not$-operator, which is defined rule-wise, translates every $\lor$-rule into a collection of ordinary rules not containing truth values in rule bodies. Hence, the space of all programs is closed under composition.

We proceed by showing that 1 is neutral with respect to composition. We first compute
\begin{align*} 
    not\ 1 = \{a\leftarrow not\ a\mid a\in A\}.
\end{align*} Next, by definition of composition, we have
\begin{align*} 
    P\circ 1 = \left\{head(r)\leftarrow body(S,N)\;\middle|\; 
    \begin{array}{l}
        S\subseteq_{sz(r_+)} 1\\
        N\subseteq_{sz(r_-)} not\ 1\\
        head(S)=body_+(r)\\
        head(N)=body_-(r)
    \end{array}\right\}.
\end{align*} Due to the simple structure of $1$ and $not\ 1$, $S\subseteq 1$ and $N\subseteq not\ 1$ imply
\begin{align*} 
    head(S)=body(S) \quad\text{and}\quad head(N)=body(horn(N)).
\end{align*} Together with 
\begin{align*} 
    head(S)=body_+(r) \quad\text{and}\quad head(N)=body_-(r)
\end{align*} we further deduce
\begin{align*} 
    body(S)=body_+(r) \quad\text{and}\quad body(horn(N))=body_-(r)
\end{align*} which finally implies
\begin{align*} 
    body(S,N)=body(r),\quad\text{for each rule $r$ in $P$}.
\end{align*} As composition is defined rule-wise by \prettyref{equ:bigcup}, this shows
\begin{align*} 
    P\circ 1=P.
\end{align*} The identity $1\circ P=P$ follows from the fact that since $1$ is Krom and Horn, we can omit every reference to $N$ in the definition of composition and $S$ amounts to a single rule $s\in P$:
\begin{align*} 
    1\circ P = \{head(r)\leftarrow body(s)\mid s\in P,\,head(s)=body(r)\}=P.
\end{align*} Hence, we have established that composition gives rise to a unital magma with neutral element 1. That the magma is ordered by set inclusion is obvious. We now turn our attention to the operation of union. In \prettyref{equ:bigcup}, we argued for the right-distributivity \prettyref{equ:P_cup_R_circ_R} of composition. That the empty set is a left zero is obvious.
\end{proof}

We will call magmas and arising from compositions of answer set programs as above \textit{\textbf{ASP magmas}}.

\subsection{Cup}\label{§:Cup}

Here we introduce the cup as an associative commutative binary operation on programs with identity (\prettyref{t:cup}), which will allow us to decompose the bodies of rules and programs into its positive and negative parts (cf. \prettyref{equ:pos_neg} and \prettyref{equ:posr_sqcup_negr}). 

\begin{definition} We define the \textit{\textbf{cup}} of $P$ and $R$ by
\begin{align*} 
    P\sqcup R:=\{head(r)\leftarrow body(r,s)\mid r\in P,s\in R, head(r)=head(s)\}.
\end{align*}
\end{definition}

For instance, we have
\begin{align*} 
    \left\{
    \begin{array}{l}
        a\leftarrow b\\
        a\leftarrow c
    \end{array}
    \right\}\sqcup \left\{
    \begin{array}{l}
        a\leftarrow b\\
        a\leftarrow c
    \end{array}
    \right\}=\left\{
    \begin{array}{l}
        a\leftarrow b\\
        a\leftarrow c\\
        a\leftarrow b,c
    \end{array}
    \right\}
\end{align*} which shows that cup is \textit{not} idempotent.

We can now decompose a rule $r$ of the form \prettyref{equ:r} in different ways, for example, as\footnote{We can omit parentheses as cup is associative according to the forthcoming \prettyref{t:cup}.}
\begin{align}\label{equ:r_sqcup} 
    \{r\}=\{a_0\leftarrow a_1\}\sqcup\ldots\sqcup\{a_0\leftarrow a_\ell\}\sqcup\{a_0\leftarrow not\ a_{\ell+1}\}\sqcup\ldots\sqcup\{a_0\leftarrow not\ a_k\}
\end{align} and as
\begin{align}\label{equ:pos_neg} 
    \{r\}=\{r_+\}\sqcup\{r_-\}.
\end{align}

As cup is defined rule-wise, we have
\begin{align}\label{equ:r_sqcup_s} 
    P\sqcup R=\bigcup_{\substack{r\in P\\s\in R}}(\{r\}\sqcup\{s\}).
\end{align}


\begin{notation} We make the notational convention that composition binds stronger than cup.
\end{notation}

\begin{fact} The space $(\mathbb P,\sqcup,A)$ of all answer set programs forms a finite commutative monoid with respect to cup with the neutral element given by the alphabet $A$, that is,
\begin{align} 
    \label{equ:cup_1} P\sqcup (Q\sqcup R) &= (P\sqcup Q)\sqcup R\\
    P\sqcup R &= R\sqcup P\\
    \label{equ:cup_3} P\sqcup A &= A\sqcup P=P.
\end{align}
\end{fact}

The next result shows that cup and union are compatible:

\begin{theorem}\label{t:cup} The space $(\mathbb P,\cup,\sqcup,\emptyset,A)$ forms a finite idempotent semiring with respect to union and cup with the zero given by the empty program. That is, we have the following identities in addition to \prettyref{equ:cup_1} -- \prettyref{equ:cup_3}:
\begin{align} 
    \emptyset\sqcup P &= P\sqcup\emptyset=\emptyset\\
    \label{equ:P_cup_R_sqcup_Q} (P\cup R)\sqcup Q &= (P\sqcup Q)\cup (R\sqcup Q)\\
    \label{equ:Q_sqcup_P_cup_R} Q\sqcup (P\cup R) &= (Q\sqcup P)\cup (Q\sqcup R).
\end{align} 
Finally, given two rules $r$ and $s$, in case $body(r)\cap body(s)=\emptyset$, we have
\begin{align}\label{equ:r_sqcup_s_Q} 
    (\{r\}\sqcup\{s\})Q&=\{r\}Q\sqcup\{s\}Q.
\end{align}
\end{theorem}
\begin{proof} The first identity holds trivially. The identities \prettyref{equ:P_cup_R_sqcup_Q} and \prettyref{equ:Q_sqcup_P_cup_R} follow from \prettyref{equ:r_sqcup_s}.

We proceed with proving \prettyref{equ:r_sqcup_s_Q} as follows. We distinguish two cases: (i) If $head(r)\neq head(s)$ then
\begin{align*} 
    (\{r\}\sqcup\{s\})Q=\emptyset Q=\emptyset
\end{align*} and since
\begin{align*} 
    head(\{r\}Q)\subseteq head(\{r\})=\{head(r)\} 
\end{align*} and
\begin{align*} 
    head(\{r\}Q)\subseteq head(\{s\})=\{head(s)\},
\end{align*} we have
\begin{align*} 
    head(\{r\}Q)\neq head(\{s\}Q)
\end{align*} which implies
\begin{align*} 
    \{r\}Q\sqcup\{s\}Q=\emptyset.
\end{align*} 

(ii) If $head(r)=head(s)$ then we have
\begin{align*} 
    (\{r\}&\sqcup\{s\})Q=\\
    &\left\{head(r)\leftarrow body(S,N) \;\middle|\; 
    \begin{array}{l}
        S\subseteq_{sz(\{r_+,s_+\})}Q\\
        N\subseteq_{sz(\{r_-,s_-\})}not\ Q\\
        head(S)=body_+(r)\cup body_+(s)\\
        head(N)=body_-(r)\cup body_-(s)
    \end{array}\right\}.
\end{align*} Now $body(r)\cap body(s)=\emptyset$ implies
\begin{align*} 
    sz(\{r_+,s_+\})=sz(r_+)+sz(s_+)
\end{align*} and
\begin{align*} 
    sz(\{r_-,s_-\})=sz(r_-)+sz(s_-).
\end{align*} Hence, the above expression is equivalent to
\begin{align*} 
    &\left\{head(r)\leftarrow body(S_r,S_s,N_r,N_s) \;\middle|\; 
    \begin{array}{l}
        S_r\subseteq_{sz(r_+)}Q\\
        S_s\subseteq_{sz(s_+)}Q\\
        N_r\subseteq_{sz(r_-)}not\ Q\\
        N_s\subseteq_{sz(s_-)}not\ Q\\
        head(S_r)=body_+(r)\\
        head(S_s)=body_+(s)\\
        head(N_r)=body_-(r)\\
        head(N_s)=body_-(s)
    \end{array}\right\}\\
    &=\left\{head(r)\leftarrow body(S,N) \;\middle|\; 
    \begin{array}{l}
        S\subseteq_{sz(r_+)}Q\\
        N\subseteq_{sz(r_-)}not\ Q\\
        head(S)=body_+(r)\\
        head(N)=body_-(r)
    \end{array}\right\}\sqcup\\
    &\qquad\qquad\left\{head(r)\leftarrow body(S,N) \;\middle|\; 
    \begin{array}{l}
        S\subseteq_{sz(s_+)}Q\\
        N\subseteq_{sz(s_-)}not\ Q\\
        head(S)=body_+(s)\\
        head(N)=body_-(s)
    \end{array}\right\}\\
    &=\{r\}Q\sqcup\{s\}Q.
\end{align*}
\end{proof}

The following counter-example shows why we require $body(r)\cap body(s)=\emptyset$ in \prettyref{equ:r_sqcup_s_Q}:
\begin{align*} 
    (\{a\leftarrow b\}\sqcup\{a\leftarrow b,c\})\circ\left\{
    \begin{array}{l}
        b\leftarrow d\\
        b\leftarrow e\\
        c\leftarrow f
    \end{array}
    \right\}= \left\{
    \begin{array}{l}
        a\leftarrow d,f\\
        a\leftarrow e,f
    \end{array}
    \right\}
\end{align*} whereas
\begin{align*} 
    \left(\{a\leftarrow b\}\circ\left\{
    \begin{array}{l}
        b\leftarrow d\\
        b\leftarrow e\\
        c\leftarrow f
    \end{array}
    \right\}\right)\sqcup\left(\{a\leftarrow b,c\}\circ\left\{
    \begin{array}{l}
        b\leftarrow d\\
        b\leftarrow e\\
        c\leftarrow f
    \end{array}
    \right\}\right)=\left\{
    \begin{array}{l}
        a\leftarrow d,f\\
        a\leftarrow e,f\\
        a\leftarrow d,e,f
    \end{array}
    \right\}.
\end{align*}


As rules can be decomposed into a positive and a negative part according to \prettyref{equ:pos_neg}, composition splits into a positive and a negative part as well.

\begin{corollary} For any answer set programs $P$ and $R$, we have
\begin{align}\label{equ:posr_sqcup_negr} PR=\bigcup_{r\in P}[\{r_+\}R\sqcup\{horn(r_-)\}not\ R].
\end{align}
\end{corollary}
\begin{proof} We compute\footnote{The forward reference in the last line is non-circular and thus harmless.}
\begin{align*} 
    PR
        &\stackrel{\prettyref{equ:bigcup}}= \bigcup_{r\in P}(\{r\}R)\\
        &\stackrel{\prettyref{equ:pos_neg}}= \bigcup_{r\in P}[(\{r_+\}\sqcup\{r_-\})R]\\
        &\stackrel{\prettyref{equ:r_sqcup_s_Q}}= \bigcup_{r\in P}[\{r_+\}R\sqcup\{r_-\}R]\\
        &\stackrel{\prettyref{equ:NR}}= \bigcup_{r\in P}[\{r_+\}R\sqcup\{horn(r_-)\}not\ R].
\end{align*}
\end{proof}

\section{Restricted classes of programs}\label{§:Restricted}

In this section, we shall study the basic properties of composition in the important classes of Horn, negative, and Krom programs, showing that the notion of composition simplifies in each of these classes.

\subsection{Krom-Horn programs}\label{§:Krom}

Recall that we call a Horn program \textit{Krom} if it contains only rules with at most one body atom. This includes interpretations, unit programs, and permutations. 


\begin{proposition} For any Krom-Horn program $K$ and answer set program $R$, composition simplifies to
\begin{align*} 
    KR=facts(K)\cup\{a\leftarrow B\mid a\leftarrow b\in K,b\leftarrow B\in R\}.
\end{align*}
\end{proposition}

We have the following structural result as a specialization of \prettyref{t:ASP}:

\begin{theorem}\label{t:Krom} The space of all Krom-Horn programs forms a monoid with the neutral element given by the unit program. Moreover, Krom-Horn programs distribute from the left, that is, for any answer set programs $P$ and $R$, we have
\begin{align}\label{equ:K_P_cup_R} 
    K(P\cup R)=(KP)\cup (KR).
\end{align} This implies that the space of proper\footnote{If a Krom-Horn program $K$ contains facts then $K\circ\emptyset=facts(K)\neq\emptyset$ violates the axiom $a\cdot 0=0$ of a semiring.} Krom-Horn programs forms a finite idempotent semiring.
Moreover, For any Krom-Horn program $K$ and answer set programs $P$ and $R$, we have
\begin{align}\label{equ:KPR} 
    K(PR)=(KP)R.
\end{align}
\end{theorem}
\begin{proof} For the proof of the first three statements see \citeA[Theorem 12]{Antic21-1}.\footnote{To be more precise, \citeA[Theorem 12]{Antic21-1} shows $K(P\cup R)=KP\cup KR$ in case $P$ and $R$ are \textit{Horn}; as the bodies of $P$ and $R$ are irrelevant, the proof can be directly transferred to our setting.}

To prove \prettyref{equ:KPR}, by \prettyref{equ:bigcup} and \prettyref{equ:P_cup_R_Q} it suffices to prove
\begin{align}\label{equ:r_PR} \{r\}(PR)=(\{r\}P)R\qquad\text{for any Krom-Horn rule $r$.}
\end{align} We distinguish two cases: (i) in case $r$ is a fact $a\in A$, we have $\{r\}(PR)=\{r\}=(\{r\}P)R$; (ii) in case $r$ is a Krom-Horn rule $a\leftarrow b$, we have $$a\leftarrow a_1,\ldots,a_\ell,not\ a_{\ell+1},\ldots not\ a_k\in \{a\leftarrow b\}(PR),\quad k\geq 0,$$ iff there is a rule
\begin{align*} 
    b\leftarrow a_1,\ldots,a_\ell,not\ a_{\ell+1},\ldots not\ a_k\in PR
\end{align*} iff there is a rule
\begin{align*} 
    b\leftarrow b_1,\ldots,b_m,not\ b_{m+1},\ldots,not\ b_n\in P
\end{align*} and there are subprograms
\begin{align*} 
    \left\{
    \begin{array}{l}
        b_1\leftarrow B_1\\
        \vdots\\
        b_m\leftarrow B_m
    \end{array}
    \right\}\subseteq R \quad\text{and}\quad \left\{
    \begin{array}{l}
        b_{m+1}\leftarrow B_{m+1}\\
        \vdots\\
        b_n\leftarrow B_n
    \end{array}
    \right\}\subseteq not\ R
\end{align*} such that
\begin{align*} 
    B_1\cup\ldots\cup B_n=\{a_1,\ldots,a_\ell,not\ a_{\ell+1},\ldots not\ a_k\}
\end{align*} iff there is a rule
\begin{align*} 
    a\leftarrow b_1,\ldots,b_m,not\ b_{m+1},\ldots,not\ b_n\in\{a\leftarrow b\}P
\end{align*} and
\begin{align*} 
    a\leftarrow (B_1\cup\ldots\cup B_n)=a\leftarrow a_1,\ldots,a_\ell,not\ a_{\ell+1},\ldots not\ a_k\in (\{a\leftarrow b\}P)R.
\end{align*} Hence, we have shown \prettyref{equ:r_PR}, from which we deduce for any Krom-Horn program $K$:
\begin{align*} 
    K(PR)\stackrel{\prettyref{equ:bigcup}}=\bigcup_{r\in K}[\{r\}(PR)]\stackrel{\prettyref{equ:r_PR}}=\bigcup_{r\in K}[(\{r\}P)R]\stackrel{\prettyref{equ:P_cup_R_Q}}=\left[\bigcup_{r\in K}(\{r\}P)\right]R\stackrel{\prettyref{equ:bigcup}}=(KP)R.
\end{align*}
\end{proof}

\subsubsection{Interpretations}

Formally, interpretations are Krom programs containing only rules with empty bodies (i.e. facts), which gives interpretations a special compositional meaning:

\begin{fact}\label{f:IP=I} Every interpretation is a left zero with respect to composition which means that for any answer set program $P$, we have
\begin{align}\label{equ:IP=I} 
    IP=I.
\end{align} Consequently, the space of interpretations forms a right ideal.\footnote{In \prettyref{c:I_ideal} we will see that it forms an ideal.}
\end{fact}

\subsubsection{Permutations}\label{§:pi}

With every permutation $\pi:A\to A$ we associate a Krom-Horn program
\begin{align*} 
    \pi=\{\pi(a)\leftarrow a\mid a\in A\}.
\end{align*} We adopt here the standard cycle notation for permutations. For instance, we have
\begin{align*} 
    \pi_{(a\,b)}:=\left\{
    \begin{array}{l}
        a\leftarrow b\\
        b\leftarrow a
    \end{array}
    \right\}\quad\text{and}\quad \pi_{(a\,b)(c)}:=\left\{
    \begin{array}{l}
        a\leftarrow b\\
        b\leftarrow a\\
        c\leftarrow c
    \end{array}
    \right\}.
\end{align*} Notice that the inverse $\pi^{-1}$ of a permutation $\pi$ translates into the dual of a program. Interestingly, we can rename the atoms occurring in a program via permutations and composition by
\begin{align*} 
    (\pi P)\pi^d=\{\pi(head(r))\leftarrow\pi(body(r))\mid r\in P\}. 
\end{align*}

We have the following structural result as a direct instance of a more general result for permutations:

\begin{proposition} The space of all permutation programs forms a subgroup of the space of all answer set programs.
\end{proposition}

\subsection{Horn programs}\label{§:Horn}

Recall that a program is called \textit{Horn} if it contains only positive rules not containing negation as failure. The composition of Horn programs has been studied by \citeA{Antic21-1} and we shall recall here some basic results.

As the syntactic structure of Horn programs consisting of rules without negation as failure in rule bodies is much simpler than the structure of general answer set programs, one can expect the composition to simplify for Horn programs. In fact, the next result shows that even in the case where only the left program in the composition is Horn, the definition of composition simplifies dramatically.

\begin{proposition}\label{p:HR} For any Horn program $H$ and answer set program $R$, composition simplifies to
\begin{align}\label{equ:HR} 
    HR=\{head(r)\leftarrow body(S)\mid r\in P,S\subseteq_{sz(r)} R, head(S)=body(r)\}.
\end{align}
\end{proposition}
\begin{proof} Since $H$ is Horn, we can omit every expression containing $N$ in \prettyref{d:circ}.
\end{proof}

We have the following structural result:

\begin{theorem}\label{t:Horn} The space of all Horn programs forms a unital submagma of the space of all answer set programs with the neutral element given by the unit program. 
\end{theorem}
\begin{proof} See \citeA[Theorem 9]{Antic21-1}. 
\end{proof}

\subsection{Negative programs}\label{§:Negative}

Recall that negative programs contain only rules with negated body atoms including facts and interpretations.\footnote{This is reasonable as we can interpret every fact $a$ as the negated ``rule'' $a\leftarrow not\ \mathbf f$ (see \prettyref{§:C}).} As negative programs are \textit{not} closed under composition, the space of negative programs does \textit{not} form a submagma of the space of all answer set programs, which is in contrast to the situation for Horn programs (cf. \prettyref{t:Horn}). This is witnessed, for example, by the identity\footnote{$a$ and $not\ not\ a$ are not equivalent under answer set semantics and the logic of here-and-there, which is irrelevant here as we are interested only in the compositional structure of $not$, that is,
\begin{align*} 
    \{a \leftarrow not\ b\}\circ \{b \leftarrow not\ c\} = \{a \leftarrow not\ not\ c\} = \{a \leftarrow c\}
\end{align*} since
\begin{align*} 
    \{a \leftarrow not\ not\ c\} \circ \{c\} = \{a\} = \{a \leftarrow c\}\circ \{c\}.
\end{align*}}
\begin{align*} 
    (not\ 1)(not\ 1)=1.
\end{align*} However, computing the composition with respect to negative programs still simplifies compared to the general case as we can reduce the composition with a negative program to the Horn case as follows.

\begin{proposition}\label{p:NR} For any negative program $N$ and answer set program $R$, composition simplifies to
\begin{align}\label{equ:NR} NR=horn(N)(not\ R).
\end{align}
\end{proposition}
\begin{proof} We compute
\begin{align*} 
    NR
        &=\left\{head(r)\leftarrow body(M) \;\middle|\; 
        \begin{array}{l}
            r\in N\\
            M\subseteq_{sz(r)} not\ R\\
            head(M)=body(horn(r))
        \end{array}\right\}\\
       &=\left\{head(r)\leftarrow body(M) \;\middle|\;
            \begin{array}{l}
                  r\in horn(N)\\
                  M\subseteq_{sz(r)} not\ R\\
                  head(M)=body(r)
            \end{array}\right\}\\
    &\stackrel{\prettyref{equ:HR}}= horn(N)(not\ R).
\end{align*}
\end{proof}


Interestingly enough, composition is compatible with negation as failure in the following sense.

\begin{proposition} For any answer set programs $P$ and $R$, we have
\begin{align}\label{equ:not_P} not\ P=(not\ 1)P \quad\text{and}\quad not(PR)=(not\ P)R.
\end{align}
\end{proposition}
\begin{proof} As a direct consequence of \prettyref{p:NR} and since $not\ 1$ is negative, we have
\begin{align*} 
    (not\ 1)P \stackrel{\prettyref{equ:NR}}= horn(not\ 1)(not\ P) = 1(not\ P) = not\ P,
\end{align*} which further implies
\begin{align*} 
    not(PR)=(not\ 1)(PR)\stackrel{\prettyref{equ:KPR}}=((not\ 1)P)R=(not\ P)R.
\end{align*}
\end{proof}



\section{Algebraic transformations}\label{§:Transformations}

In this section, we study algebraic transformations of programs expressible via composition and other operations.

Our first observation is that we can compute the heads and bodies of a Horn program $H$ by
\begin{align*} 
    head(H)=HA \quad\text{and}\quad body(H) = head(proper(H)^d) = proper(H)^dA.
\end{align*} This implies that we can compute the heads of an answer set program $P$ by
\begin{align*} 
    head(P) = P_+A\cup horn(P_-)A.
\end{align*} Consequently, we can compute the heads of a negative program $N$ by
\begin{align*} 
    head(N) = N\emptyset.
\end{align*}

\subsection{Reducts}

Reducing the rules of a program to a restricted alphabet is a fundamental operation on programs and in this section we will show how reducts can be computed via composition, cup, and union (cf. \prettyref{t:reducts}).

\subsubsection{Horn programs}

We first recall some results of \citeA{Antic21-1} on computing the reducts of Horn programs.

\begin{proposition} The left and right reducts of a Horn program $H$ with respect to some interpretation $I$ can be expressed as
\begin{align}\label{equ:^IH} 
    ^IH = (1^I)H \quad\text{and}\quad H^I = H(1^I).
\end{align} Consequently, we obtain the reduction of $H$ to the atoms in $I$ via
\begin{align}\label{equ:H_I} 
    {^I}(H^I) = (^IH)^I = (1^I)H(1^I).
\end{align} 
\end{proposition}
\begin{proof} See the proof of \citeA[Proposition 17]{Antic21-1}.
\end{proof}

We can compute the facts of a Horn program $H$ via
\begin{align}\label{equ:facts_H} 
    H^\emptyset \stackrel{\prettyref{equ:^IH}}= H(1^\emptyset) = H\emptyset = facts(H).
\end{align}

Moreover, for any interpretations $I$ and $J$, we have
\begin{align}\label{equ:^JI} ^JI=I\cap J \quad\text{and}\quad I^J=I.
\end{align}

\begin{proposition}\label{p:^IH} For any Horn programs $H$ and $G$, we have
\begin{align} 
\label{equ:gH^I} gH^I &= H\\
\label{equ:^IHG} ^I(H\cup G) &= {^I}H\cup{^I}G \quad\text{and}\quad {^I}(HG) = ({^I}H)G\\
\label{equ:^IH_cup_G} (H\cup G)^I &= H^I\cup G^I \quad\text{and}\quad (HG)^I = HG^I.
\end{align}
\end{proposition}
\begin{proof} The first identity holds trivially. For the identities in the last two lines, see the proof of \citeA[Proposition 17]{Antic21-1}.
\end{proof}

\subsubsection{Negative programs}

Reducts of negative programs are in a sense ``dual'' to reducts of Horn programs studied above. In the rest of this section, $N$ denotes a negative program. 

Our first observation is a dual of \prettyref{equ:^IH}.

\begin{proposition} The left and right reducts of a negative program $N$ with respect to an interpretation $I$ can be expressed as
\begin{align}\label{equ:N^I} 
    ^IN = (1^I)N \quad\text{and}\quad N^I = horn(N)1^{A-I}.
\end{align} Moreover, the Gelfond-Lifschitz reduct of $N$ with respect to $I$ can be expressed as
\begin{align}\label{equ:gN^I} 
    gN^I = NI.
\end{align}
\end{proposition}
\begin{proof} The proof of the first identity is analogous to the proof of \prettyref{equ:^IH}. For the second identity, we compute
\begin{align*} 
    N^I=horn(N)^{A-I} \stackrel{\prettyref{equ:^IH}}= horn(N)1^{A-I}.
\end{align*} For the last identity, we compute
\begin{align*} 
    NI
        &= \left\{head(r)\leftarrow body(M) \;\middle|\; 
            \begin{array}{l}
                r\in N\\ 
                M\subseteq_{sz(r_-)}not\ I\\
                head(M)=body_-(r)
            \end{array}\right\}\\
        &= \left\{head(r)\leftarrow body(M) \;\middle|\; 
            \begin{array}{l}
                r\in N\\ 
                M\subseteq_{sz(r)}A-I\\
                head(M)=body_-(r)
            \end{array}\right\}\\
        & = \{head(r)\mid r\in N, body_-(r)\subseteq A-I\}\\
        & = \{head(r)\mid r\in N, I\cap body_-(r)=\emptyset\}\\
        & = \{head(r)\mid r\in N, I\models body(r)\}\\
        & = gN^I,
\end{align*} where the second equality follows from $r_-=r$ as $r$ is negative, and the third equality follows from $body(M)=\emptyset$ as $M$ is a subset of $A-I$.
\end{proof}


\subsubsection{Answer set programs}

We now focus on reducts of arbitrary answer set programs. Our first observation is that the first identity in \prettyref{equ:^IH} can be lifted to the general case as
\begin{align*} 
    ^IP = (1^I)P.
\end{align*}


The next lemma shows that reducts are compatible with cup and union.

\begin{lemma} For any answer set program $P$ and interpretation $I$, we have
\begin{align} 
    \label{equ:gP_cup_R^I} g(P\cup R)^I&=gP^I\cup gR^I \quad\text{and}\quad g(P\sqcup R)^I=gP^I\sqcup gR^I\\
    \label{equ:^IP_cup_R}^I(P\cup R)&={^I}P\cup{^I}R \quad\text{and}\quad ^I(P\sqcup R)={^IP}\sqcup{^IR}\\
    \label{equ:P_cup_R^I} (P\cup R)^I&=P^I\cup R^I \quad\text{and}\quad (P\sqcup R)^I=P^I\sqcup R^I.
\end{align}
\end{lemma}
\begin{proof} The identities on the left-hand side are immediate consequences of the rule-wise defintion of reducts. 

For the identities on the right-hand side, we first show that for any rules $r$ and $s$, we have
\begin{align}\label{equ:gr_sqcup_s^I} 
    g(\{r\}\sqcup\{s\})^I=g\{r\}^I\sqcup g\{s\}^I.
\end{align} For this, we distinguish two cases: (i) if $head(r)\neq head(s)$ then 
\begin{align*} 
    g(\{r\}\sqcup\{s\})^I=g\emptyset^I=\emptyset=g\{r\}^I\sqcup g\{s\}^I;
\end{align*} (ii) if $head(r)=head(s)$ then 
\begin{align*} 
    g(\{r\}\sqcup\{s\})^I&=g\{head(r)\leftarrow body(r,s)\}^I=\{\{r_+,s_+\}\mid I\models body_-(r,s)\}\\
    &=\{r_+\mid I\models body_-(r)\}\sqcup\{s_+\mid I\models body_-(s)\}\\
    &=g\{r\}^I\sqcup g\{s\}^I.
\end{align*} Now we have
\begin{align*} 
    g(P\sqcup R)^I&\stackrel{\prettyref{equ:r_sqcup_s}}=g\left(\bigcup_{\substack{r\in P\\s\in R}}(\{r\}\sqcup\{s\})\right)^I\stackrel{\prettyref{equ:gP_cup_R^I}}=\bigcup_{\substack{r\in P\\s\in R}}g(\{r\}\sqcup\{s\})^I\\
    &\stackrel{\prettyref{equ:gr_sqcup_s^I}}=\bigcup_{\substack{r\in P\\s\in R}}(g\{r\}^I\sqcup g\{s\}^I)=\bigcup_{\substack{r\in gP^I\\s\in gR^I}}(\{r\}\sqcup\{s\})\stackrel{\prettyref{equ:r_sqcup_s}}=gP^I\sqcup gR^I.
\end{align*} The proofs of the remaining identities are analogous.

\end{proof}

We are now ready to express the Gelfond-Lifschitz and Faber-Leone-Pfeifer reducts via composition, cup, and union.

\begin{theorem}\label{t:reducts} For any answer set program $P$ and interpretation $I$, we have
\begin{align*} 
    gP^I=\bigcup_{r\in P}[\{r_+\}\sqcup\{r_-\}I]
\end{align*} and
\begin{align*} 
    P^I=\bigcup_{r\in P}\left(\{r_+\}1^I\sqcup\{horn(r_-)\}1^{A-I}\right).
\end{align*}
\end{theorem}
\begin{proof} We first compute, for any rule $r$,
\begin{align*} 
    g\{r\}^I\stackrel{\prettyref{equ:gP_cup_R^I}}=g\{r_+\}^I\sqcup g\{r_-\}^I\stackrel{\prettyref{equ:gH^I},\prettyref{equ:gN^I}}=\{r_+\}\sqcup\{r_-\}I,
\end{align*} extended to any answer set program $P$ by
\begin{align*} 
    gP^I\stackrel{\prettyref{equ:gP_cup_R^I}}=\bigcup_{r\in P}g\{r\}^I=\bigcup_{r\in P}[\{r_+\}\sqcup\{r_-\}I].
\end{align*} For the Faber-Leone-Pfeifer reduct, we have
\begin{align*} 
    P^I&=\bigcup_{r\in P}\{r\}^I=\bigcup_{r\in P}(\{r_+\}\sqcup\{r_-\})^I\stackrel{\prettyref{equ:P_cup_R^I}}=\bigcup_{r\in P}\left(\{r_+\}^I\sqcup\{r_-\}^I\right)\\
    & \stackrel{\prettyref{equ:^IH},\prettyref{equ:N^I}}=\bigcup_{r\in P}\left(\{r_+\}1^I\sqcup\{horn(r_-)\}1^{A-I}\right).
\end{align*}
\end{proof}

\subsection{Adding and removing body literals}

We now want to study algebraic transformations of rule bodies.

\subsubsection{Horn programs}

We shall first recall here some basic constructions and results concerning the manipulation of Horn programs by following the lines of \citeA{Antic21-1}. 

In what follows, $H$ denotes a Horn program. 

For example, we have
\begin{align*} 
    \{a\leftarrow b,c\}\left\{
        \begin{array}{l}
            b \leftarrow b\\
            c
        \end{array}
        \right\} = \{a\leftarrow b\}.
\end{align*} The general construction here is that we add a tautological rule $b\leftarrow b$ for every body atom $b$ of $H$ which we want to preserve, and we add a fact $c$ in case we want to remove $c$ from the rule bodies in $H$. 

\begin{definition} For any interpretation $I$, define
\begin{align*} 
    I^\ominus:=1^{A-I}\cup I.
\end{align*} Notice that $.^\ominus$ is computed with respect to some fixed alphabet $A$.
\end{definition}
 
For instance, we have
\begin{align*} 
    A^\ominus=A \quad\text{and}\quad\emptyset^\ominus=1.
\end{align*} 

Interestingly enough, we have
\begin{align*} 
    I^\ominus I=(1^{A-I}\cup I)I\stackrel{\prettyref{equ:P_cup_R_Q}}=1^{A-I}I\cup I^2\stackrel{\prettyref{equ:IP=I},\prettyref{equ:^IH},\prettyref{equ:^JI}}=((A-I)\cap I)\cup I=I
\end{align*} and
\begin{align*} 
    I^\ominus H=(1^{A-I}\cup I)H\stackrel{\prettyref{equ:bigcup}}=1^{A-I}H\cup IH \stackrel{\prettyref{equ:IP=I},\prettyref{equ:^IH}}={^{A-I}}H\cup I.
\end{align*}

In the example above, we have
\begin{align*} 
    \{c\}^\ominus= \left\{
    \begin{array}{l}
        a\leftarrow a\\
        b\leftarrow b\\
        c
    \end{array}
    \right\} \quad\text{and}\quad\{a\leftarrow b,c\}\{c\}^\ominus=\{a\leftarrow b\}
\end{align*} as desired. Notice also that the facts of a program are, of course, not affected by composition on the right (cf. \prettyref{equ:facts_PR}), that is, we cannot expect to remove facts via composition on the right.\footnote{However, notice that we can \textit{add} facts via composition on the left via $P\cup I=(1\cup I)P$ (cf. \prettyref{equ:P_cup_I}).}

We have the following general result:

\begin{fact}\label{f:ominus} For any Horn program $H$ and interpretation $I$, we have
\begin{align*} 
    HI^\ominus=\{head(r)\leftarrow (body(r)-I)\mid r\in H\}.
\end{align*}
\end{fact}


In analogy to the above construction, we can add body literals to Horn programs via composition on the right. For example, we have
\begin{align*} 
    \{a\leftarrow b\}\{b\leftarrow b,not\ c\}=\{a\leftarrow b,not\ c\}.
\end{align*} Here, the general construction is as follows. 

\begin{definition} For any set of literals $B$, define
\begin{align*} 
    B^\oplus:=\{a\leftarrow(\{a\}\cup B)\mid a\in A\}.
\end{align*} Notice that $.^\oplus$ is computed with respect to some fixed alphabet $A$.
\end{definition}

For instance, we have
\begin{align*} 
    A^\oplus=\{a\leftarrow A\mid a\in A\} \quad\text{and}\quad\emptyset^\oplus=1.
\end{align*} Interestingly enough, we have for any interpretation $I$,
\begin{align*} 
    I^\oplus I^\ominus=I^\ominus \quad\text{and}\quad I^\oplus I=I.
\end{align*} Moreover, in the example above, we have
\begin{align*} 
    \{not\ c\}^\oplus= \left\{
        \begin{array}{l}
            a\leftarrow a,not\ c\\
            b\leftarrow b,not\ c\\
            c\leftarrow not\ c
        \end{array}
        \right\} \quad\text{and}\quad\{a\leftarrow b\}\{not\ c\}^\oplus=\{a\leftarrow b,not\ c\}
\end{align*} as desired. As composition on the right does not affect the facts of a program, we cannot expect to append body literals to facts via composition on the right. However, we can add arbitrary literals to \textit{all} proper rule bodies simultaneously and, in analogy to \prettyref{f:ominus}, we have the following general result:

\begin{fact} For any Horn program $H$ and set of literals $B$, we have
\begin{align*} 
    HB^\oplus=facts(H)\cup\{head(r)\leftarrow (body(r)\cup B)\mid r\in proper(H)\}.
\end{align*}
\end{fact}

The following example illustrates the interplay between the above concepts:

\begin{example} Consider the Horn programs
\begin{align*} 
    H= \left\{
    \begin{array}{l}
        c\\
        a\leftarrow b,c\\
        b\leftarrow a,c
    \end{array}
    \right\} \quad\text{and}\quad \pi_{(a\;b)}= \left\{
    \begin{array}{l}
        a\leftarrow b\\
        b\leftarrow a
    \end{array}
    \right\}.
\end{align*} Roughly, we obtain $H$ from $\pi_{(a\;b)}$ by adding the fact $c$ to $\pi_{(a\;b)}$ and to each body rule in $\pi_{(a\;b)}$. Conversely, we obtain $\pi_{(a\;b)}$ from $H$ by removing the fact $c$ from $H$ and by removing the body atom $c$ from each rule in $H$. This can be formalized as (here, we define $\{c\}^\ast := 1\cup\{c\}$ which yields $\{c\}^\ast\pi_{(a\;b)}=\pi_{(a\;b)}\cup\{c\}$; see the forthcoming equation \prettyref{equ:I^ast})
\begin{align*} 
    H=(\{c\}^\ast\pi_{(a\;b)})\{c\}^\oplus \quad\text{and}\quad\pi_{(a\;b)}=(1^{\{a,b\}}H)\{c\}^\ominus. 
\end{align*}
\end{example}

\subsubsection{Negative programs}

Removing body literals from negative programs is similar to the Horn case above. For example, if we want to remove the literal $not\ c$ from the rule body of $a\leftarrow not\ b,not\ c$, we compute
\begin{align*} 
    \{a\leftarrow not\ b,not\ c\} 1^{\{a,b,c\}-\{c\}}
        &\stackrel{\prettyref{equ:NR}}= horn(\{a\leftarrow not\ b,not\ c\})(not\ 1^{\{a,b\}})\\
       &=\{a\leftarrow b,c\}\left\{
            \begin{array}{l}
                a\leftarrow not\ a\\
                b\leftarrow not\ b\\
                c
            \end{array}
            \right\}\\
       &=\{a\leftarrow not\ b\}.
\end{align*}

We have the following general result:

\begin{proposition}\label{p:N} For any negative program $N$ and interpretation $I$, we have
\begin{align*} 
    N1^{A-I} = \{head(r)\leftarrow (body(r)-\{not\ a\mid a\in I\})\mid r\in N\}.
\end{align*}
\end{proposition}
\begin{proof} We compute
\begin{align*} 
    N1^{A-I}
        &\stackrel{\prettyref{equ:NR}}= horn(N)(not\ 1^{A-I})\\
        &= horn(N)(\{a\leftarrow not\ a\mid a\in A-I\}\cup I)\\
        &= \{head(r)\leftarrow ((body(r)-I)\cup\{not\ a\mid a\in (A-I)\cap body(r)\})\mid r\in horn(N)\}\\
        &= \{head(r)\leftarrow\{not\ a\mid a\in (A-I)\cap body(r)\}\mid r\in horn(N)\}\\
        &= \{head(r)\leftarrow (body(r)-\{not\ a\mid a\in I\})\mid r\in N\}.
\end{align*}
\end{proof}

Adding literals to bodies of negative programs via composition is not possible as composition with negative rules yields disjunctions in rule bodies as is demonstrated by the following simple computation:
\begin{align*} \{a\leftarrow not\ b\}\{b\leftarrow b,c\}= \left\{
\begin{array}{l}
    a\leftarrow not\ b\\
    a\leftarrow not\ c
\end{array}
\right\}.
\end{align*}

\subsubsection{Answer set programs}\label{§:Transformation_ASP}

Unfortunately, systematically transforming the bodies of arbitrary programs requires more refined algebraic techniques in which the positive and negative parts of rules and programs can be manipulated separately, and we shall leave this problem as future work (cf. \prettyref{§:Conclusion}).

\section{Algebraic semantics}\label{§:Algebraic}

In this section, we reformulate the fixed point semantics of answer set programs in terms of composition without any explicit reference to operators. Our key observation is that the van Emden-Kowalski immediate consequence operator of a program can be algebraically represented via composition (\prettyref{t:T_P}), which implies an algebraic characterization of the answer set semantics (\prettyref{t:AS}).

\subsection{The van Emden-Kowalski operator}

\prettyref{t:semantics} emphasizes the central role of the van Emden-Kowalski operator in answer set programming and the next result shows that it can be syntactically represented in terms of composition.

\begin{theorem}\label{t:T_P} For any answer set program $P$ and interpretation $I$, we have
\begin{align}\label{equ:T_P} 
    T_P(I) = PI.
\end{align} Moreover, we have
\begin{align*} 
    T_P\cup T_R = T_{P\cup R} \quad\text{and}\quad T_P\circ T_R = T_{P\circ R} \quad\text{and}\quad T_\emptyset = \emptyset\quad\text{and}\quad T_1 = Id_{2^A}.
\end{align*}
\end{theorem}
\begin{proof} All of the equations follow immediately from the definition of composition. For example, we compute
\begin{align*} 
    PI
        &= \left\{head(r)\leftarrow body(S,N) \;\middle|\;
            \begin{array}{l} 
                r\in P\\
                S\subseteq_{sz(r_+)}I\\
                N\subseteq_{sz(r_-)}not\ I\\
                head(S) = body_+(r)\\
                head(N) = body_-(r)
            \end{array}\right\}\\
        &= \{head(r)\mid r\in P, body_+(r)\subseteq I, body_-(r)\subseteq A-I\}\\
        &= \{head(r)\mid r\in P, I\models body(r)\}\\
        &= T_P(I),
\end{align*} where the second equality follows from
\begin{align*} 
    not\ I = A-I \quad\text{and}\quad body(S) = body(N) = \emptyset.
\end{align*}
\end{proof}

As a direct consequence of Theorems \ref{t:semantics} and \ref{t:T_P}, we have the following algebraic characterization of (supported) models in terms of composition:

\begin{corollary} An interpretation $I$ is a model of $P$ iff $PI\subseteq I$, and $I$ is a supported model of $P$ iff $PI=I$.
\end{corollary}


\begin{corollary}\label{c:I_ideal} The space of all interpretations forms an ideal.
\end{corollary}
\begin{proof} By \prettyref{f:IP=I}, we know that the space of interpretations forms a right ideal and \prettyref{t:T_P} implies that it is a left ideal --- hence, it forms an ideal.
\end{proof}

\begin{corollary} For any answer set program $P$ and interpretation $I$, we have $PI = IP$ iff $I$ is a supported model of $P$.
\end{corollary}
\begin{proof} A direct consequence of \prettyref{f:IP=I} and \prettyref{t:T_P}.
\end{proof}


\subsection{Answer sets}\label{§:AS}

We interpret programs according to their answer set semantics and since answer sets can be constructively computed by bottom-up iterations of the associated van Emden-Kowalski operators (cf. \prettyref{t:semantics}), we can finally reformulate the fixed point semantics of answer set programs in terms of sequential composition (\prettyref{t:AS}).

\begin{definition} Define the \textit{\textbf{Kleene star}} of a Horn program $H$ by
\begin{align*} 
    H^\ast := \bigcup_{n\geq 0} H^n,
\end{align*} where
\begin{align*} 
    H^n := (((HH)H)\ldots H)H\quad (n\text{ times}).
\end{align*} Moreover, define the \textit{\textbf{omega}} operation by
\begin{align*} 
    H^\omega := H^\ast\emptyset \stackrel{\prettyref{equ:facts_H}}= facts(H^\ast).
\end{align*}
\end{definition}

Notice that the unions in the computation of Kleene star are finite since $H$ is finite. For instance, for any interpretation $I$, we have as a consequence of \prettyref{f:IP=I},
\begin{align}\label{equ:I^ast} 
    I^\ast = 1\cup I\quad\text{and}\quad I^\omega = I.
\end{align} Interestingly enough, we can add the atoms in $I$ to $P$ via
\begin{align}\label{equ:P_cup_I} 
    P\cup I \stackrel{\prettyref{equ:IP=I}}= P\cup IP \stackrel{\prettyref{equ:P_cup_R_Q}}= (1\cup I)P \stackrel{\prettyref{equ:I^ast}}= I^\ast P.
\end{align} Hence, as a consequence of \prettyref{equ:P} and \prettyref{equ:P_cup_I}, we can decompose $P$ as
\begin{align*} 
    P = facts(P)^\ast\circ proper(P),
\end{align*} which, roughly, says that we can sequentially separate the facts from the proper rules in $P$.

We have the following algebraic characterization of least models due to \citeA{Antic21-1}:

\begin{theorem}\label{t:LM_H} For any Horn program $H$, we have
\begin{align*} 
    LM(H) = H^\omega.
\end{align*} Consequently, two Horn programs $H$ and $G$ are equivalent iff $H^\omega=G^\omega$.
\end{theorem}

We have finally arrived at the following algebraic characterization of answer sets in terms of composition, cup, and union (cf. \prettyref{t:reducts}):

\begin{theorem}\label{t:AS} An interpretation $I$ is an answer set of a program $P$ iff $I=(gP^I)^\omega$.
\end{theorem}
\begin{proof} The interpretation $I$ is an answer set of $P$ iff $I$ is the least model of $gP^I$ iff $I$ is the least fixed point of $T_{gP^I}$ (\prettyref{t:semantics}). By \prettyref{t:T_P}, we have
\begin{align*} 
    T_{gP^I}(J) = gP^IJ\quad\text{for every interpretation $J$.}
\end{align*} Hence, as $gP^I$ is Horn, \prettyref{t:LM_H} implies that $I$ is an answer set of $P$ iff $I=(gP^I)^\omega$.
\end{proof}

\begin{corollary}\label{c:strong} Two answer set programs $P$ and $R$ are strongly equivalent iff
\begin{align*} 
    I = (gP^I\cup gQ^I)^\omega \quad\Leftrightarrow\quad I = (gR^I\cup gQ^I)^\omega
\end{align*} holds for any interpretation $I$ and program $Q$.
\end{corollary}
\begin{proof} A direct consequence of \prettyref{equ:gP_cup_R^I} and \prettyref{t:AS}.
\end{proof}

\begin{corollary}\label{c:uniform} Two answer set programs $P$ and $R$ are uniformly equivalent iff
\begin{align*} 
    I = (g(J^\ast P)^I)^\omega \quad\Leftrightarrow\quad I = (g(J^\ast R)^I)^\omega\quad\text{for any interpretations $I$ and $J$.}
\end{align*}
\end{corollary}
\begin{proof} A direct consequence of \prettyref{equ:P_cup_I} and \prettyref{t:AS}.
\end{proof}

\section{Index and period}\label{§:Index}

In the above sections, we have \textit{algebraically} reformulated well-known concepts like the least model semantics of Horn programs and uniform and strong equivalence of answer set programs. In this section, we demonstrate the benefits of having such an algebraic description of the structure of answer set programs by transferring a purely algebraic concept into the domain of answer set programming.

For a given program $P$, consider the programs $[P]:=\{P,P^2,P^3,\ldots\}$ generated by $P$. If there are no repetitions in the list $P^2,P^3,\ldots$, then $([P],\circ)$ is isomorphic to $(\mathbb N,+)$. In this case, we say that $P$ has \textit{infinite order}. Notice that this case can only occur for \textit{infinite} programs (i.e. infinite propositional programs or first-order programs with an infinite grounding which are both not considered in this paper). Otherwise, if there are repetitions among the powers of $P$, then the set
\begin{align*} 
    \{m\in \mathbb N \mid P^m=P^n,\text{ for some $m\neq n$}\}
\end{align*} is non-empty and thus has a least element which we will call the \textit{\textbf{index}} of $P$ denoted by $index(P)$. Then the set
\begin{align*} 
    \left\{r\in \mathbb N \;\middle|\; P^{index(P)+r} = P^{index(P)}\right\}
\end{align*} is also non-empty and has a least element which we will call the \textit{\textbf{period}} of $P$ denoted by $period(P)$. Intuitively, the index and period of a program contains information about the cyclicality of a program as we shall demonstrate in the rest of this section.

\begin{fact} Every idempotent program has period 1 and index 1. This holds in particular for every interpretation.
\end{fact}
\begin{proof} The first statement is obvious and the second statement is a direct consequence of \prettyref{equ:IP=I}.
\end{proof}

\subsection{Elevators}

In this section, we shall construct programs with index $m$ and period 1 (cf. \prettyref{p:E_m}).

For every $m\geq 1$, define the \textit{\textbf{$m$-elevator}} by the Krom program
\begin{align*} 
    E_m := \{1\}\cup \{a+1\leftarrow a\mid 1\leq a<m\} = \left\{
    \begin{array}{l}
        1\\
        2\leftarrow 1\\
        \quad\vdots\\
        m\leftarrow {m-1}
    \end{array}
    \right\}.
\end{align*} For example, we have
\begin{align*} 
    E_1=\{1\},\quad E_2= \left\{
    \begin{array}{l}
        1\\
        2\leftarrow 1
    \end{array}
    \right\},\quad E_3= \left\{
    \begin{array}{l}
        1\\
        2\leftarrow 1\\
        3\leftarrow 2
    \end{array}
    \right\},\quad\ldots.
\end{align*} Notice that
\begin{align*} 
    E_m\subseteq E_n,\quad\text{for all $m\leq n$.}
\end{align*} 

For every $1\leq k\leq m$, we have
\begin{align*} 
    E_m^k=\{1,\ldots,k\}\cup \{a+1\leftarrow a+1-k\mid k\leq a\leq m-1\}
\end{align*} which shows that for all $k,\ell\leq m-1$,
\begin{align*} 
    E_m^k=E_m^\ell \quad\Rightarrow\quad k=\ell
\end{align*} and
\begin{align*} 
    E_m^m=[m].
\end{align*} For example, for
\begin{align*} 
    E_4= \left\{
    \begin{array}{l}
        1\\
        2\leftarrow 1\\
        3\leftarrow 2\\
        4\leftarrow 3
    \end{array}
    \right\}
\end{align*} we obtain the powers
\begin{align*} 
    E_4,\quad E_4^2= \left\{
    \begin{array}{l}
        1\\
        2\\
        3\leftarrow 1\\
        4\leftarrow 2
    \end{array}
    \right\},\quad E_4^3= \left\{
    \begin{array}{l}
        1\\
        2\\
        3\\
        4\leftarrow 1
    \end{array}
    \right\},\quad E_4^4= \left\{
    \begin{array}{l}
        1\\
        2\\
        3\\
        4
    \end{array}
    \right\}.
\end{align*}

Hence, we have shown:

\begin{proposition}\label{p:E_m} The $m$-elevator $E_m$ has index $m$ and period 1.
\end{proposition}

\subsection{Permutations}

Recall from \prettyref{§:pi}, that with every permutation $\pi:[n]\to [n]$,\footnote{For $n\in \mathbb N$, define $[n] := \{1,\ldots,n\}$.} $n\geq 1$, we associate the \textit{\textbf{permutation}}
\begin{align*} 
    \pi:=\{\pi(i)\leftarrow i\mid i\in [n]\}.
\end{align*} We use the well-known cycle notation so that
\begin{align*} 
    \pi_{(1\;\ldots\;n)}= \left\{
    \begin{array}{l}
        1\leftarrow n\\
        2\leftarrow 1\\
        3\leftarrow 2\\
        \vdots\\
        n\leftarrow {n-1}
    \end{array}
    \right\}.
\end{align*} In what follows, we shall write $\pi_n$ instead of $\pi_{(1\;\ldots\;n)}$ and we call $\pi_n$ the \textit{\textbf{$n$-permutation}}. The \textit{\textbf{identity $n$-permutation}} is denoted by
\begin{align*} 
    \mathbbm 1_n := \{i \leftarrow i\mid i\in [n]\}.
\end{align*} Of course,
\begin{align*} 
    index( \mathbbm 1_n) = period( \mathbbm 1_n)=1.
\end{align*} We always have
\begin{align*} 
    \pi_n^n = \mathbbm 1_n.
\end{align*}

Notice that we have
\begin{align*} 
    E_m=(\pi_m-\{1\leftarrow m\})\cup \{1\}.
\end{align*} That is, $E_m$ and $\pi_m$ differ only in the rule with head atom $1$.

\begin{fact}\label{f: 230905-pi_n} The $n$-permutation $\pi_n$ has index 1 and period $n$. Moreover, all powers $\pi_n^k$, $k\geq 1$, have index 1 and period $n$.
\end{fact}

\begin{example} The permutation program $\pi_{(1\;2)}$ has index 1 and period 2 since
\begin{align*} 
    \pi_{(1\;2)}^2 = \mathbbm 1_{\{1,2\}} \quad\text{and}\quad \pi_{(1\;2)}^3 = \pi_{(1\;2)}.
\end{align*}
\end{example}

\subsubsection{$K_{m,n}$}

We now wish to construct a Krom-Horn program $K_{m,n}$ with index $m$ and period $n$. For this, let
\begin{align*} 
    K_{m,n} := E_m\cup \pi_{(m+1\;\ldots\;m+n)}.
\end{align*} It is important to notice that we have
\begin{align}\label{equ:E_m_cap_pi=0} 
    E_m\cap \pi_{(m+1\;\ldots\;m+n)}=\emptyset.
\end{align}

\begin{theorem} The program $K_{m,n}$ has index $m$ and period $n$.
\end{theorem}
\begin{proof} Since $K_{m,n}$ is Krom, we have by \prettyref{t:Krom} (we write $\pi$ instead of $\pi_{(m+1\;\ldots\;m+n)}$)
\begin{align*} 
    K_{m,n}^2=(E_m\cup \pi)(E_m\cup \pi)=E_m^2\cup \pi E_m\cup E_m\pi\cup \pi^2=E_m^2\cup \pi^2,
\end{align*} where the last identity follows from
\begin{align*} 
    E_m\pi\stackrel{\prettyref{equ:E_m_cap_pi=0}}=\{1\}\subseteq E_m^2 \quad\text{and}\quad \pi E_m\stackrel{\prettyref{equ:E_m_cap_pi=0}}=\emptyset.
\end{align*} Notice that we have
\begin{align}\label{equ: 230905-E_m_pi^k} 
    E_m\pi^k\stackrel{\prettyref{equ:E_m_cap_pi=0}}=\{1\}\subseteq E_m^\ell \quad\text{and}\quad \pi^k E_m\stackrel{\prettyref{equ:E_m_cap_pi=0}}=\emptyset,\quad\text{for all $1\leq\ell\leq k$.}
\end{align} Another iteration yields
\begin{align*} 
    K_{m,n}^3=(E_m\cup \pi)(E_m^2\cup \pi^2)=E_m^3\cup E_m\pi^2\cup \pi E_m^2\cup \pi^3 \stackrel{\prettyref{equ: 230905-E_m_pi^k}}=E_m^3\cup \pi^3.
\end{align*} This leads us to the following general formula which is shown by a straightforward induction using the identities in \prettyref{equ: 230905-E_m_pi^k}:
\begin{align*} 
    K_{m,n}^k=E_m^k\cup \pi_{(m+1\;\ldots\;m+n)}^k.
\end{align*}

For any $1\leq k,\ell\leq m-1$, we have
\begin{align*} 
    E_m^k\neq E_m^\ell
\end{align*} and since
\begin{align*} 
    E_m^k\cap \pi=E_m^\ell\cap \pi=\emptyset,
\end{align*} this shows
\begin{align*} 
    K_{m,n}^k\neq K_{m,n}^\ell.
\end{align*} Moreover, we have
\begin{align*} 
    K_{m,n}^m=[m]\cup \pi^m.
\end{align*} and thus
\begin{align*} 
    K_{m,n}^{m+i}=[m]\cup \pi^{m+i}
\end{align*} which by \prettyref{f: 230905-pi_n} implies that $n$ is the least positive integer such that
\begin{align*} 
    K_{m,n}^{m+n}=[m]\cup \pi^{m+n}=[m]\cup \pi^m=K_{m,n}^m.
\end{align*} Hence,
\begin{align*} 
    period(K_{m,n}^m)=period(\pi^m)=period(\pi)=n.
\end{align*} In total we have thus shown
\begin{align*} 
    index(K_{m,n})=m \quad\text{and}\quad period(K_{m,n})=n.
\end{align*}
\end{proof}

\subsection{Negative programs}

In this section, we will see that the cyclicality of negative programs is quite different from the positive ones which is due to the alternating nature of (double) negation.

\subsubsection{Negative elevators}

The \textit{\textbf{negative $m$-elevator}} is given by
\begin{align*} 
    not\;E_m=\{2,\ldots,m\}\cup \{a+1 \leftarrow not\; a \mid 1\leq a\leq m-1\}=\left\{
    \begin{array}{l}
        2\\
        \vdots\\
        m\\
        2\leftarrow not\; 1\\
        \vdots\\
        m\leftarrow not\;(m-1)
    \end{array}
    \right\}.
\end{align*} For every $1\leq k\leq m$, we have
\begin{align*} 
    (not\;E_m)^k=\begin{cases}
        \{2,\ldots,m\}\cup \{a+1 \leftarrow a-k+1\mid k\leq a\leq m-1\} & \text{if $k$ is even,}\\
        \{2,\ldots,m\}\cup \{a+1 \leftarrow not\;(a-k+1)\mid k\leq a\leq m-1\} & \text{if $k$ is odd.}\\
    \end{cases} 
\end{align*} For example, we have
\begin{align*} 
    (not\;E_m)^2=\left\{
    \begin{array}{l}
        2\\
        \vdots\\
        m\\
        3\leftarrow 1\\
        4 \leftarrow 2\\
        \vdots\\
        m\leftarrow {m-2}
    \end{array}
    \right\},\quad (not\;E_m)^3= \left\{
    \begin{array}{l}
        2\\
        \vdots\\
        m\\
        4 \leftarrow not\; 1\\
        \vdots\\
        m \leftarrow not\;(m-3)
    \end{array}
    \right\},\quad (not\;E_m)^4= \left\{
    \begin{array}{l}
        2\\
        \vdots\\
        m\\
        5 \leftarrow 1\\
        \vdots\\
        m \leftarrow m-4
    \end{array}
    \right\}.
\end{align*} In particular, we have
\begin{align*} 
    (not\;E_m)^m=[m]=E_m^m,
\end{align*} regardless of whether $m$ is even or odd.

We have thus shown:

\begin{proposition} The negative $m$-elevator $not\;E_m$ has index $m$ and period 1, that is,
\begin{align*} 
    index(not\;E_m)=index(E_m)=m \quad\text{and}\quad period(not\;E_m)=period(E_m)=1.
\end{align*}
\end{proposition}

\subsubsection{Negative permutations}

The \textit{\textbf{negative $n$-permutation}} is given by
\begin{align*} 
    not\;\pi_n= \left\{
    \begin{array}{l}
        2 \leftarrow not\;1\\
        \vdots\\
        n \leftarrow not\;(n-1)\\
        1 \leftarrow not\;n
    \end{array}
    \right\}.
\end{align*} We have
\begin{align}\label{equ: 230905-not_pi_n^k} 
    (not\;\pi_n)^k=\begin{cases}
        \pi_n^k & \text{if $k$ is even},\\
        not\;\pi_n^k & \text{if $k$ is odd}.
    \end{cases}
\end{align} In particular, we have
\begin{align}\label{equ: 230905-not_pi_n} 
    (not\;\pi_n)^n=\begin{cases}
        \mathbbm 1_n & \text{if $n$ is even},\\
        not\;\mathbbm 1_n & \text{if $n$ is odd}.
    \end{cases}
\end{align} We thus have
\begin{align*} 
    \text{$n$ is even} \quad\Rightarrow\quad period(not\;\pi_n)=n.
\end{align*} What if $n$ is odd? We claim:
\begin{align*} 
    \text{$n$ is odd} \quad\Rightarrow\quad period(not\;\pi_n)=2n.
\end{align*} 

We have
\begin{align*} 
    (not\; \mathbbm 1_n)^2= \mathbbm 1_n^2= \mathbbm 1_n,
\end{align*} which by \prettyref{equ: 230905-not_pi_n} implies
\begin{align*} 
    (not\;\pi_n)^{2n}=\mathbbm 1_n
\end{align*} and thus
\begin{align*} 
    (not\;\pi_n)^{1+2n}=not\;\pi_n.
\end{align*} 

It remains to show
\begin{align}\label{equ: 230905-not_pi_n_ldots_1} 
    not\;\pi_n\neq (not\;\pi_n)^2\neq\ldots\neq (not\;\pi_n)^{2n}.
\end{align} Since $n$ is odd by assumption, by \prettyref{equ: 230905-not_pi_n^k} this sequence of powers equals
\begin{align}\label{equ: 230905-not_pi_n_ldots_2} 
    not\;\pi_n,\quad \pi_n^2,\quad\ldots\quad not\;\pi_n^{n-2},\quad \pi_n^{n-1},\quad not\;\pi_n^n,\quad \pi_n^{n+1},\quad not\;\pi_n^{n+2},\quad\ldots\quad not\;\pi_n^{2n-1},\quad \pi_n^{2n}.
\end{align} Notice that since $n$ is odd, we have, for all $1\leq k\leq n$,
\begin{align*} 
    (not\;\pi_n)^{n+k}=\begin{cases}
        \pi_n^k & \text{if $k$ is odd},\\
        not\;\pi_n^k & \text{if $k$ is even}.
    \end{cases}
\end{align*} In particular, we have
\begin{align*} 
    not\;\pi_n^{2n}=\pi_n^n.
\end{align*} Hence, the sequence in \prettyref{equ: 230905-not_pi_n_ldots_2} can be written as
\begin{align*} 
    not\;\pi_n,\quad \pi_n^2,\quad\ldots\quad not\;\pi_n^{n-2},\quad \pi_n^{n-1},\quad not\;\pi_n^n,\quad \pi_n,\quad not\;\pi_n^2,\quad\ldots\quad not\;\pi_n^{n-1},\quad \pi_n^n.
\end{align*} This immediately shows \prettyref{equ: 230905-not_pi_n_ldots_1}.

For example, for $n=3$, the $2n=6$ distinct powers of $not\;\pi_3$ are given by
\begin{align*}
    (not\;\pi_3)^1&= \left\{
    \begin{array}{l}
        2 \leftarrow not\;1\\
        3 \leftarrow not\;2\\
        1 \leftarrow not\;3
    \end{array}
    \right\},\quad (not\;\pi_3)^2=\pi_3^2= \left\{
    \begin{array}{l}
        2 \leftarrow 3\\
        3 \leftarrow 1\\
        1 \leftarrow 2
    \end{array}
    \right\},\quad (not\;\pi_3)^3=not\;\pi_3^3= \left\{
    \begin{array}{l}
        2 \leftarrow not\;2\\
        3 \leftarrow not\;3\\
        1 \leftarrow not\;1
    \end{array}
    \right\}\\
    (not\;\pi_3)^4&=\pi_3^4= \left\{
    \begin{array}{l}
        2 \leftarrow 1\\
        3 \leftarrow 2\\
        1 \leftarrow 3
    \end{array}
    \right\},\quad (not\;\pi_3)^5=not\;\pi_3^5= \left\{
    \begin{array}{l}
        2 \leftarrow not\;3\\
        3 \leftarrow not\;1\\
        1 \leftarrow not\;2
    \end{array}
    \right\},\quad (not\;\pi_3)^6=\pi_3^6= \left\{
    \begin{array}{l}
        2 \leftarrow 2\\
        3 \leftarrow 3\\
        1 \leftarrow 1
    \end{array}
    \right\}.
\end{align*}


In total, we have thus shown:

\begin{proposition} The negative $n$-permutation $not\;\pi_n$ has index 1 and period
\begin{align*} 
    period(not\;\pi_n)=\begin{cases}
        n & \text{if $n$ is even},\\
        2n & \text{if $n$ is odd}.
    \end{cases}
\end{align*} That is,
\begin{align*} 
    period(not\;\pi_n)=\begin{cases}
        period(\pi_n) & \text{if $n$ is even},\\
        2\times period(\pi_n) & \text{if $n$ is odd}.
    \end{cases}
\end{align*}
\end{proposition}

\begin{remark} Notice that the period of $not\;\pi_n$ is always even!
\end{remark}

This means that the cyclicality of $not\;\pi_n$ is identical to that of $\pi_n$ iff $n$ is even, and it differs by a factor of 2 iff $n$ is odd; this shows that positive and negative programs are different when it comes to cycles which coincides with our intuition.

\section{Aperiodic programs}\label{§:Aperiodic}

In this brief final technical section of the paper we show how the algebraization of answer set programming provided in the previous sections can help to find interesting novel classes of programs. The following definition is a simple instantiation of the concept of an aperiodic element of a semigroup\footnote{See e.g. \citeA{Howie03}.} in the ASP-setting:

\begin{definition} A program $P$ is \textit{\textbf{aperiodic}} iff $P^{n+1}=P^n$ for some $n\geq 1$.
\end{definition}

Given the intuitive relationship between the period of a program and its cyclicality (see \prettyref{§:Index}), we can say that aperiodic programs are in a sense ``acyclic'' and in fact it is easy to see that acyclic Horn programs are aperiodic (but see \prettyref{e:not_aperiodic}):

\begin{fact} Every acyclic propositional Horn program is aperiodic.
\end{fact} 

The following example shows that the class of aperiodic Horn programs strictly contains the acyclic ones:

\begin{example}\label{e:not_aperiodic} The program
\begin{align*} 
    P := \left\{
    \begin{array}{l}
        a\\
        b \leftarrow a\\  
        b \leftarrow a,b
    \end{array}
    \right\}
\end{align*} satisfies
\begin{align*} 
    P^3 = P^2 = \left\{
    \begin{array}{l}
        a\\
        b\\
        b \leftarrow a\\
        b \leftarrow a,b
    \end{array}
    \right\}
\end{align*} and thus is aperiodic despite containing the irrelevant cyclic rule $b \leftarrow a,b$.
\end{example}

Given the counterexample to an acyclic program in \prettyref{e:not_aperiodic}, it seems to us that the notion of an aperiodic program better captures the intuition of an ``acyclic'' program than the usual definition in terms of level mappings.

The least model of an aperiodic Horn program is particularly simple to compute:

\begin{fact} For an aperiodic Horn program $H$ with $H^{n+1}=H^n$, we have $H^\omega=facts(H^n)$.
\end{fact}

\section{Future work}\label{§:FW}

In the future, we plan to extend the constructions and results of this paper to wider classes of answer set programs as, for example, \textbf{higher-order programs} \cite{Miller12} and \textbf{disjunctive programs} \cite{Eiter97}. The former task is non-trivial since function symbols require most general unifiers in the definition of composition and give rise to \textit{infinite} algebras, whereas disjunctive rules yield non-deterministic behavior which is more difficult to handle algebraically. Nonetheless, we expect interesting results to follow in all of the aforementioned cases.

In \prettyref{§:Aperiodic} we have introduced the class of \textbf{aperiodic answer set programs} strictly containing the acyclic ones and it is interesting to study this class of programs from a purely mathematical and from a practical point of view.

From an artificial intelligence perspective, it is interesting to apply the abstract algebraic framework of analogical proportions of the form ``$a$ is to $b$ as $c$ is to $d$'' in \citeA{Antic22} to ASP algebras as defined in this paper for \textbf{answer set program synthesis by analogy}; \citeA{Antic23-23} did exactly that for first-order Horn logic programs. For this it will be of central importance to study (sequential) \textbf{decompositions} of various program classes. Specifically, we wish to compute decompositions of arbitrary answer set programs (and extensions thereof) into ``prime'' programs, where we expect permutations (\prettyref{§:pi}) to play a fundamental role in such decompositions. For this, it will be necessary to resolve the issue of a \textbf{``prime'' or indecomposable answer set program}. Algebraically, it will be of central importance to study Green's relations \cite{Green51} in the finite ASP magmas and algebras introduced in this paper. From a practical point of view, a mathematically satisfactory theory of program decompositions is relevant to modular knowledge representation and optimization of reasoning.

Corollaries \ref{c:strong} and \ref{c:uniform} are the entry point to an algebraization of the notions of strong \cite{Lifschitz01} and uniform equivalence \cite{Eiter03}, related to \citeS{Truszczynski06} operational characterization. This line of research is related to modular answer set programming \cite{Oikarinen06a} and in the future we wish to express algebraic operations for modular program constructions within our framework.

Approximation Fixed Point Theory (AFT) \cite{Denecker12,Denecker04} is an operational framework based on \citeS{Fitting02} work on fixed points in logic programming relating different semantics of logical formalisms with non-monotonic entailment \cite<cf.>{Pelov07,Antic13,Antic20}. It is interesting to try to \textbf{\textit{syntactically} reformulate AFT} within the ASP algebra introduced in this paper. For this, the first step will be to redefine approximations as pairs of programs instead of lattice elements satisfying certain conditions. In \prettyref{§:Transformation_ASP}, we mentioned systematic algebraic transformations of arbitrary answer set programs for which the current tools are not sufficient. More precisely, manipulating rule bodies requires a finer separation of the positive and negative parts of rules and programs in the vein of approximations in ``syntactic AFT'' mentioned before. 


Finally, in the future we wish to compare the sequential composition of answer set programs to the \textbf{cascade product} as introduced in \citeA{Antic14} from an algebraic point of view.

\section{Conclusion}\label{§:Conclusion}

This paper contributed to the foundations of answer set programming by introducing and studying the sequential composition of answer set programs. We showed in our main structural result (\prettyref{t:ASP}) that the space of all programs forms a finite unital magma with respect to composition ordered by set inclusion, which distributes from the right over union. We called the magmas induced by sequential composition \textit{ASP magmas}, and we called the algebras induced by sequential composition and union \textit{ASP algebras}. Moreover, we showed that the restricted class of Krom programs is distributive and therefore its proper instance forms an idempotent semiring (\prettyref{t:Krom}). These results extended the results of \citeA{Antic21-1} from Horn to answer set programs. From a logical point of view, we obtained an algebraic meta-calculus for reasoning about answer set programs. Algebraically, we obtained a correspondence between answer set programs and finite magmas, which enables us to transfer algebraic concepts to the logical setting. We have introduced the index and period of a program as an algebraic measure of its cyclicality and obtained the novel class of aperiodic programs strictly generalizing the acyclic programs. In a broader sense, this paper is a further step towards an algebra of rule-based logical theories and we expect interesting concepts and results to follow.

\bibliographystyle{theapa}
\bibliography{/Users/christianantic/Bibdesk/Bibliography,/Users/christianantic/Bibdesk/Publications_J,/Users/christianantic/Bibdesk/Publications_C,/Users/christianantic/Bibdesk/Preprints,/Users/christianantic/Bibdesk/Submitted,/Users/christianantic/Bibdesk/Notes}

\begin{thebibliography}{}

\bibitem[\protect\BCAY{Anti\'c}{Anti\'c}{2014}]{Antic14}
Anti\'c, C. \BBOP2014\BBCP.
\newblock \BBOQ On cascade products of answer set programs\BBCQ\
\newblock {\Bem Theory and Practice of Logic Programming}, {\Bem 14\/}(4-5),
  711--723.
\newblock \url{https://doi.org/10.1017/S1471068414000301}.

\bibitem[\protect\BCAY{Anti\'c}{Anti\'c}{2020}]{Antic20}
Anti\'c, C. \BBOP2020\BBCP.
\newblock \BBOQ Fixed point semantics for stream reasoning\BBCQ\
\newblock {\Bem Artificial Intelligence}, {\Bem 288}, 103370.
\newblock \url{https://doi.org/10.1016/j.artint.2020.103370}.

\bibitem[\protect\BCAY{Anti\'c}{Anti\'c}{2022}]{Antic22}
Anti\'c, C. \BBOP2022\BBCP.
\newblock \BBOQ Analogical proportions\BBCQ\
\newblock {\Bem Annals of Mathematics and Artificial Intelligence}, {\Bem
  90\/}(6), 595--644.
\newblock \url{https://doi.org/10.1007/s10472-022-09798-y}.

\bibitem[\protect\BCAY{Anti\'c}{Anti\'c}{2023}]{Antic23-23}
Anti\'c, C. \BBOP2023\BBCP.
\newblock \BBOQ Logic program proportions\BBCQ\
\newblock {\Bem Annals of Mathematics and Artificial Intelligence}.
\newblock \url{https://doi.org/10.1007/s10472-023-09904-8}.

\bibitem[\protect\BCAY{Anti\'c}{Anti\'c}{2024}]{Antic21-1}
Anti\'c, C. \BBOP2024\BBCP.
\newblock \BBOQ Sequential composition of propositional logic programs\BBCQ\
\newblock {\Bem Annals of Mathematics and Artificial Intelligence}, {\Bem
  92\/}(2), 505--533.
\newblock \url{https://doi.org/10.1007/s10472-024-09925-x}.

\bibitem[\protect\BCAY{Anti\'c, Eiter,\ \BBA\ Fink}{Anti\'c
  et~al.}{2013}]{Antic13}
Anti\'c, C., Eiter, T., \BBA\ Fink, M. \BBOP2013\BBCP.
\newblock \BBOQ {HEX} semantics via approximation fixpoint theory\BBCQ\
\newblock In Cabalar, P.\BBACOMMA\  \BBA\ Son, T.~C.\BEDS, {\Bem LPNMR 2013},
  \BPGS\ 102--115.
\newblock \url{https://doi.org/10.1007/978-3-642-40564-8_11}.

\bibitem[\protect\BCAY{Antoniou}{Antoniou}{1997}]{Antoniou97}
Antoniou, G. \BBOP1997\BBCP.
\newblock {\Bem Nonmonotonic Reasoning}.
\newblock Massachusetts Institute of Technology.

\bibitem[\protect\BCAY{Apt}{Apt}{1990}]{Apt90}
Apt, K.~R. \BBOP1990\BBCP.
\newblock \BBOQ Logic programming\BBCQ\
\newblock In van Leeuwen, J.\BED, {\Bem Handbook of Theoretical Computer
  Science}, \lowercase{\BVOL}~B, \BPGS\ 493--574. Elsevier, Amsterdam.

\bibitem[\protect\BCAY{Apt\ \BBA\ Bezem}{Apt\ \BBA\ Bezem}{1991}]{Apt91}
Apt, K.~R.\BBACOMMA\  \BBA\ Bezem, M. \BBOP1991\BBCP.
\newblock \BBOQ Acyclic programs\BBCQ\
\newblock {\Bem New Generation Computing}, {\Bem 9\/}(3-4), 335--363.

\bibitem[\protect\BCAY{Arbib}{Arbib}{1968}]{Arbib68}
Arbib, M.~A.\BED. \BBOP1968\BBCP.
\newblock {\Bem {Algebraic Theory of Machines, Languages, and Semigroups}}.
\newblock Academic Press, New York.

\bibitem[\protect\BCAY{Baral}{Baral}{2003}]{Baral03}
Baral, C. \BBOP2003\BBCP.
\newblock {\Bem {Knowledge Representation, Reasoning and Declarative Problem
  Solving}}.
\newblock Cambridge University Press, Cambridge.

\bibitem[\protect\BCAY{Bossi, Bugliesi, Gabbrielli, Levi,\ \BBA\ Meo}{Bossi
  et~al.}{1996}]{Bossi96}
Bossi, A., Bugliesi, M., Gabbrielli, M., Levi, G., \BBA\ Meo, M.
  \BBOP1996\BBCP.
\newblock \BBOQ Differential logic programs: Programming methodologies and
  semantics\BBCQ\
\newblock {\Bem Science of Computer Programming}, {\Bem 27}, 217--262.

\bibitem[\protect\BCAY{Brewka, Eiter,\ \BBA\ Truszczynski}{Brewka
  et~al.}{2011}]{Brewka11}
Brewka, G., Eiter, T., \BBA\ Truszczynski, M. \BBOP2011\BBCP.
\newblock \BBOQ {Answer set programming at a glance}\BBCQ\
\newblock {\Bem Communications of the ACM}, {\Bem 54\/}(12), 92--103.

\bibitem[\protect\BCAY{Brogi, Lamma,\ \BBA\ Mello}{Brogi
  et~al.}{1992a}]{Brogi92}
Brogi, A., Lamma, E., \BBA\ Mello, P. \BBOP1992a\BBCP.
\newblock \BBOQ Compositional model-theoretic semantics for logic
  programs\BBCQ\
\newblock {\Bem New Generation Computing}, {\Bem 11}, 1--21.

\bibitem[\protect\BCAY{Brogi, Mancarella, Pedreschi,\ \BBA\ Turini}{Brogi
  et~al.}{1992b}]{Brogi92a}
Brogi, A., Mancarella, P., Pedreschi, D., \BBA\ Turini, F. \BBOP1992b\BBCP.
\newblock \BBOQ Meta for modularising logic programming\BBCQ\
\newblock In {\Bem META 1992}, \BPGS\ 105--119.

\bibitem[\protect\BCAY{Brogi, Mancarella, Pedreschi,\ \BBA\ Turini}{Brogi
  et~al.}{1999}]{Brogi99}
Brogi, A., Mancarella, P., Pedreschi, D., \BBA\ Turini, F. \BBOP1999\BBCP.
\newblock \BBOQ Modular logic programming\BBCQ\
\newblock {\Bem ACM Transactions on Programming Languages and Systems}, {\Bem
  16\/}(4), 1361--1398.

\bibitem[\protect\BCAY{Brogi\ \BBA\ Turini}{Brogi\ \BBA\
  Turini}{1995}]{Brogi95}
Brogi, A.\BBACOMMA\  \BBA\ Turini, F. \BBOP1995\BBCP.
\newblock \BBOQ Fully abstract compositional semantics for an algebra of logic
  programs\BBCQ\
\newblock {\Bem Theoretical Computer Science}, {\Bem 149\/}(2), 201--229.

\bibitem[\protect\BCAY{Bugliesi, Lamma,\ \BBA\ Mello}{Bugliesi
  et~al.}{1994}]{Bugliesi94}
Bugliesi, M., Lamma, E., \BBA\ Mello, P. \BBOP1994\BBCP.
\newblock \BBOQ Modularity in logic programming\BBCQ\
\newblock {\Bem The Journal of Logic Programming}, {\Bem 19-20\/}(1), 443--502.

\bibitem[\protect\BCAY{Chen, Kifer,\ \BBA\ Warren}{Chen et~al.}{1993}]{Chen93}
Chen, W., Kifer, M., \BBA\ Warren, D.~S. \BBOP1993\BBCP.
\newblock \BBOQ Hi{L}og: {A} foundation for higher-order logic
  programming\BBCQ\
\newblock {\Bem The Journal of Logic Programming}, {\Bem 15\/}(3), 187--230.

\bibitem[\protect\BCAY{Clark}{Clark}{1978}]{Clark78}
Clark, K.~L. \BBOP1978\BBCP.
\newblock \BBOQ {Negation as failure}\BBCQ\
\newblock In Gallaire, H.\BBACOMMA\  \BBA\ Minker, J.\BEDS, {\Bem {Logic and
  Data Bases}}, \BPGS\ 293--322. Plenum Press, New York.

\bibitem[\protect\BCAY{Dao-Tran, Eiter, Fink,\ \BBA\ Krennwallner}{Dao-Tran
  et~al.}{2009}]{DaoTran09}
Dao-Tran, M., Eiter, T., Fink, M., \BBA\ Krennwallner, T. \BBOP2009\BBCP.
\newblock \BBOQ Modular nonmonotonic logic programming revisited\BBCQ\
\newblock In {\Bem ICLP 2009}, LNCS 5649, \BPGS\ 145--159. Springer-Verlag,
  Berlin/Heidelberg.

\bibitem[\protect\BCAY{Denecker, Bruynooghe,\ \BBA\ Vennekens}{Denecker
  et~al.}{2012}]{Denecker12}
Denecker, M., Bruynooghe, M., \BBA\ Vennekens, J. \BBOP2012\BBCP.
\newblock \BBOQ Approximation fixpoint theory and the semantics of logic and
  answer set programs\BBCQ\
\newblock In Erdem, E., Lee, J., Lierler, Y., \BBA\ Pearce, D.\BEDS, {\Bem
  Correct Reasoning}, \lowercase{\BVOL}\ 7265 of {\Bem LNCS}, \BPGS\ 178--194,
  Heidelberg. Springer-Verlag.

\bibitem[\protect\BCAY{Denecker, Marek,\ \BBA\ Truszczy{\'n}ski}{Denecker
  et~al.}{2004}]{Denecker04}
Denecker, M., Marek, V., \BBA\ Truszczy{\'n}ski, M. \BBOP2004\BBCP.
\newblock \BBOQ Ultimate approximation and its application in nonmonotonic
  knowledge representation systems\BBCQ\
\newblock {\Bem Information and Computation}, {\Bem 192\/}(1), 84--121.

\bibitem[\protect\BCAY{Dong\ \BBA\ Ginsburg}{Dong\ \BBA\
  Ginsburg}{1990}]{Dong90}
Dong, G.\BBACOMMA\  \BBA\ Ginsburg, S. \BBOP1990\BBCP.
\newblock \BBOQ On the decomposition of datalog program mappings\BBCQ\
\newblock {\Bem Theoretical Computer Science}, {\Bem 76\/}(1), 143--177.

\bibitem[\protect\BCAY{Eiter\ \BBA\ Fink}{Eiter\ \BBA\ Fink}{2003}]{Eiter03}
Eiter, T.\BBACOMMA\  \BBA\ Fink, M. \BBOP2003\BBCP.
\newblock \BBOQ Uniform equivalence of logic programs under the stable model
  semantics\BBCQ\
\newblock In {\Bem ICLP 2003}, LNCS 2916, \BPGS\ 224--238. Springer-Verlag.

\bibitem[\protect\BCAY{Eiter, Gottlob,\ \BBA\ Mannila}{Eiter
  et~al.}{1997}]{Eiter97}
Eiter, T., Gottlob, G., \BBA\ Mannila, H. \BBOP1997\BBCP.
\newblock \BBOQ Disjunctive datalog\BBCQ\
\newblock {\Bem ACM Transactions on Database Systems}, {\Bem 22\/}(3),
  364--418.

\bibitem[\protect\BCAY{Eiter, Ianni,\ \BBA\ Krennwallner}{Eiter
  et~al.}{2009}]{Eiter09}
Eiter, T., Ianni, G., \BBA\ Krennwallner, T. \BBOP2009\BBCP.
\newblock \BBOQ {Answer set programming: a primer}\BBCQ\
\newblock In {\Bem Reasoning Web. Semantic Technologies for Information
  Systems, {\em volume 5689 of} Lecture Notes in Computer Science}, \BPGS\
  40--110. Springer, Heidelberg.

\bibitem[\protect\BCAY{Eiter, Ianni, Lukasiewicz, Schindlauer,\ \BBA\
  Tompits}{Eiter et~al.}{2008}]{Eiter08a}
Eiter, T., Ianni, G., Lukasiewicz, T., Schindlauer, R., \BBA\ Tompits, H.
  \BBOP2008\BBCP.
\newblock \BBOQ {Combining answer set programming with description logics for
  the Semantic Web}\BBCQ\
\newblock {\Bem Artificial Intelligence}, {\Bem 172\/}(12-13), 1495--1539.

\bibitem[\protect\BCAY{Eiter, Ianni, Schindlauer,\ \BBA\ Tompits}{Eiter
  et~al.}{2005}]{Eiter05}
Eiter, T., Ianni, G., Schindlauer, R., \BBA\ Tompits, H. \BBOP2005\BBCP.
\newblock \BBOQ A uniform integration of higher-order reasoning and external
  evaluations in answer-set programming\BBCQ\
\newblock In Kaelbling, L.~P.\BBACOMMA\  \BBA\ Saffiotti, A.\BEDS, {\Bem IJCAI
  2005}, \BPGS\ 90--96.

\bibitem[\protect\BCAY{Faber, Leone,\ \BBA\ Pfeifer}{Faber
  et~al.}{2004}]{Faber04}
Faber, W., Leone, N., \BBA\ Pfeifer, G. \BBOP2004\BBCP.
\newblock \BBOQ {Recursive aggregates in disjunctive logic programs: semantics
  and complexity}\BBCQ\
\newblock In Alferes, J.\BBACOMMA\  \BBA\ Leite, J.\BEDS, {\Bem JELIA 2004},
  LNCS 3229, \BPGS\ 200--212. Springer, Berlin.

\bibitem[\protect\BCAY{Faber, Pfeifer,\ \BBA\ Leone}{Faber
  et~al.}{2011}]{Faber11}
Faber, W., Pfeifer, G., \BBA\ Leone, N. \BBOP2011\BBCP.
\newblock \BBOQ Semantics and complexity of recursive aggregates in answer set
  programming\BBCQ\
\newblock {\Bem Artificial Intelligence}, {\Bem 175\/}(1), 278--298.

\bibitem[\protect\BCAY{Fichte, Truszczy{\'n}ski,\ \BBA\ Woltran}{Fichte
  et~al.}{2015}]{Fichte15}
Fichte, J., Truszczy{\'n}ski, M., \BBA\ Woltran, S. \BBOP2015\BBCP.
\newblock \BBOQ Dual-normal logic programs - the forgotten class\BBCQ\
\newblock {\Bem Theory and Practice of Logic Programming}, {\Bem 15\/}(4-5),
  496--510.

\bibitem[\protect\BCAY{Fitting}{Fitting}{2002}]{Fitting02}
Fitting, M. \BBOP2002\BBCP.
\newblock \BBOQ Fixpoint semantics for logic programming --- a survey\BBCQ\
\newblock {\Bem Theoretical Computer Science}, {\Bem 278\/}(1-2), 25--51.

\bibitem[\protect\BCAY{Gaifman\ \BBA\ Shapiro}{Gaifman\ \BBA\
  Shapiro}{1989}]{Gaifman89}
Gaifman, H.\BBACOMMA\  \BBA\ Shapiro, E. \BBOP1989\BBCP.
\newblock \BBOQ Fully abstract institute of mathematics fully abstract
  compositional semantics for logic programs\BBCQ\
\newblock In {\Bem Proceedings of the 16th Annual ACM Symposium on Principles
  of Programming Languages}, \BPGS\ 134--142. ACM Press.

\bibitem[\protect\BCAY{Gebser, Kaufmann,\ \BBA\ Schaub}{Gebser
  et~al.}{2012}]{Gebser12}
Gebser, M., Kaufmann, B., \BBA\ Schaub, T. \BBOP2012\BBCP.
\newblock \BBOQ {Conflict-driven answer set solving: from theory to
  practice}\BBCQ\
\newblock {\Bem Artificial Intelligence}, {\Bem 187-188\/}(C), 52--89.

\bibitem[\protect\BCAY{Gelfond\ \BBA\ Lifschitz}{Gelfond\ \BBA\
  Lifschitz}{1991}]{Gelfond91}
Gelfond, M.\BBACOMMA\  \BBA\ Lifschitz, V. \BBOP1991\BBCP.
\newblock \BBOQ Classical negation in logic programs and disjunctive
  databases\BBCQ\
\newblock {\Bem New Generation Computing}, {\Bem 9\/}(3-4), 365--385.

\bibitem[\protect\BCAY{Ginzburg}{Ginzburg}{1968}]{Ginzburg68}
Ginzburg, A. \BBOP1968\BBCP.
\newblock {\Bem {Algebraic Theory of Automata}}.
\newblock ACM Monograph Series. Academic Press, New York/London.

\bibitem[\protect\BCAY{Giunchiglia, Lierler,\ \BBA\ Maratea}{Giunchiglia
  et~al.}{2006}]{Giunchiglia06}
Giunchiglia, E., Lierler, Y., \BBA\ Maratea, M. \BBOP2006\BBCP.
\newblock \BBOQ {Answer set programming based on propositional
  satisfiability}\BBCQ\
\newblock {\Bem Journal of Automated Reasoning}, {\Bem 36\/}(4), 345--377.

\bibitem[\protect\BCAY{Green}{Green}{1951}]{Green51}
Green, J.~A. \BBOP1951\BBCP.
\newblock \BBOQ On the structure of semigroups\BBCQ\
\newblock {\Bem Annals of Mathematics}, {\Bem 54\/}(1), 163--172.

\bibitem[\protect\BCAY{Hill\ \BBA\ Lloyd}{Hill\ \BBA\ Lloyd}{1994}]{Hill94}
Hill, P.\BBACOMMA\  \BBA\ Lloyd, J.~W. \BBOP1994\BBCP.
\newblock {\Bem The G{\"o}del Programming Language}.
\newblock The MIT Press.

\bibitem[\protect\BCAY{Howie}{Howie}{2003}]{Howie03}
Howie, J.~M. \BBOP2003\BBCP.
\newblock {\Bem {Fundamentals of Semigroup Theory}}.
\newblock London Mathematical Society Monographs New Series. Oxford University
  Press, Oxford.

\bibitem[\protect\BCAY{Janhunen}{Janhunen}{2006}]{Janhunen06}
Janhunen, T. \BBOP2006\BBCP.
\newblock \BBOQ Some (in)translatability results for normal logic programs and
  propositional theories\BBCQ\
\newblock {\Bem Journal of Applied Non-Classical Logics}, {\Bem 16\/}(1),
  35--86.

\bibitem[\protect\BCAY{Krom}{Krom}{1967}]{Krom67}
Krom, M.~R. \BBOP1967\BBCP.
\newblock \BBOQ The decision problem for a class of first-order formulas in
  which all disjunctions are binary\BBCQ\
\newblock {\Bem Mathematical Logic Quarterly}, {\Bem 13\/}(1-2), 15--20.

\bibitem[\protect\BCAY{Leinster}{Leinster}{2004}]{Leinster04}
Leinster, T. \BBOP2004\BBCP.
\newblock {\Bem Higher Operads, Higher Categories}, \lowercase{\BVOL}\ 298 of
  {\Bem London Mathematical Society Lecture Note Series}.
\newblock Cambridge University Press, Cambridge.

\bibitem[\protect\BCAY{Leone, Pfeifer, Faber, Eiter, Gottlob, Perri,\ \BBA\
  Scarcello}{Leone et~al.}{2006}]{Leone06}
Leone, N., Pfeifer, G., Faber, W., Eiter, T., Gottlob, G., Perri, S., \BBA\
  Scarcello, F. \BBOP2006\BBCP.
\newblock \BBOQ {The DLV system for knowledge representation and
  reasoning}\BBCQ\
\newblock {\Bem ACM Transactions on Computational Logic}, {\Bem 7\/}(3),
  499--562.

\bibitem[\protect\BCAY{Lifschitz}{Lifschitz}{2002}]{Lifschitz02}
Lifschitz, V. \BBOP2002\BBCP.
\newblock \BBOQ {Answer set programming and plan generation}\BBCQ\
\newblock {\Bem Artificial Intelligence}, {\Bem 138}, 39--54.

\bibitem[\protect\BCAY{Lifschitz}{Lifschitz}{2010}]{Lifschitz10}
Lifschitz, V. \BBOP2010\BBCP.
\newblock \BBOQ Thirteen definitions of a stable model\BBCQ\
\newblock In Blass, A., Derschowitz, N., \BBA\ Reisig, W.\BEDS, {\Bem Gurevich
  Festschrift, LNCS 6300}, \BPGS\ 488 -- 503. Springer-Verlag.

\bibitem[\protect\BCAY{Lifschitz}{Lifschitz}{2019}]{Lifschitz19}
Lifschitz, V. \BBOP2019\BBCP.
\newblock {\Bem Answer Set Programming}.
\newblock Springer Nature Switzerland AG, Cham, Switzerland.

\bibitem[\protect\BCAY{Lifschitz, Pearce,\ \BBA\ Valverde}{Lifschitz
  et~al.}{2001}]{Lifschitz01}
Lifschitz, V., Pearce, D., \BBA\ Valverde, A. \BBOP2001\BBCP.
\newblock \BBOQ Strongly equivalent logic programs\BBCQ\
\newblock {\Bem ACM Transactions on Computational Logic}, {\Bem 2\/}(4),
  526--541.

\bibitem[\protect\BCAY{Lloyd}{Lloyd}{1987}]{Lloyd87}
Lloyd, J.~W. \BBOP1987\BBCP.
\newblock {\Bem {Foundations of Logic Programming}\/} (2 \BEd).
\newblock Springer-Verlag, Berlin, Heidelberg.

\bibitem[\protect\BCAY{Maher}{Maher}{1988}]{Maher88}
Maher, M.~J. \BBOP1988\BBCP.
\newblock \BBOQ Equivalences of logic programs\BBCQ\
\newblock In Minker, J.\BED, {\Bem Foundations of Deductive Databases and Logic
  Programming}, \BCH~16, \BPGS\ 627--658. Morgan Kaufmann Publishers.

\bibitem[\protect\BCAY{Mancarella\ \BBA\ Pedreschi}{Mancarella\ \BBA\
  Pedreschi}{1988}]{Mancarella88}
Mancarella, P.\BBACOMMA\  \BBA\ Pedreschi, D. \BBOP1988\BBCP.
\newblock \BBOQ An algebra of logic programs\BBCQ\
\newblock In Kowalski, R.\BBACOMMA\  \BBA\ Bowen, K.~A.\BEDS, {\Bem Proceedings
  of the 5th International Conference on Logic Programming}, \BPGS\ 1006--1023.
  The MIT Press, Cambridge MA.

\bibitem[\protect\BCAY{Marek\ \BBA\ Truszczy{\'n}ski}{Marek\ \BBA\
  Truszczy{\'n}ski}{1999}]{Marek99}
Marek, V.\BBACOMMA\  \BBA\ Truszczy{\'n}ski, M. \BBOP1999\BBCP.
\newblock \BBOQ {Stable models and an alternative logic programming
  paradigm}\BBCQ\
\newblock In Apt, K.~R., Marek, V., Truszczy{\'n}ski, M., \BBA\ Warren,
  D.~S.\BEDS, {\Bem The Logic Programming Paradigm: a 25-Year Perspective},
  \BPGS\ 375--398. Springer, Berlin.

\bibitem[\protect\BCAY{Miller\ \BBA\ Nadathur}{Miller\ \BBA\
  Nadathur}{2012}]{Miller12}
Miller, D.\BBACOMMA\  \BBA\ Nadathur, G. \BBOP2012\BBCP.
\newblock {\Bem Programming with Higher-Order Logic}.
\newblock Cambridge University Press.

\bibitem[\protect\BCAY{Niemel{\"a}, Simons,\ \BBA\ Soininen}{Niemel{\"a}
  et~al.}{1999}]{Niemela99}
Niemel{\"a}, I., Simons, P., \BBA\ Soininen, T. \BBOP1999\BBCP.
\newblock \BBOQ {Stable model semantics of weight constraint rules}\BBCQ\
\newblock In Gelfond, M., Leone, N., \BBA\ Pfeifer, G.\BEDS, {\Bem Proceedings
  of the 5th International Conference on Logic Programming and Nonmonotonic
  Reasoning (LPNMR 1999), {\em volume 1730 of} Lecture Notes in Computer
  Science}, \BPGS\ 317--331. Springer-Verlag, Berlin.

\bibitem[\protect\BCAY{Oikarinen}{Oikarinen}{2006}]{Oikarinen06a}
Oikarinen, E. \BBOP2006\BBCP.
\newblock {\Bem Modular Answer Set Programming}.
\newblock Ph.D.\ thesis, Helsinki University of Technology, Helsinki, Finland.

\bibitem[\protect\BCAY{Oikarinen\ \BBA\ Janhunen}{Oikarinen\ \BBA\
  Janhunen}{2006}]{Oikarinen06}
Oikarinen, E.\BBACOMMA\  \BBA\ Janhunen, T. \BBOP2006\BBCP.
\newblock \BBOQ Modular equivalence for normal logic programs\BBCQ\
\newblock In Brewka, G., Coradeschi, S., Perini, A., \BBA\ Traverso, P.\BEDS,
  {\Bem Proc. 17th European Conference on Artificial Intelligence}, \BPGS\
  412--416. IOS Press, Amsterdam, Netherlands.

\bibitem[\protect\BCAY{O'Keefe}{O'Keefe}{1985}]{OKeefe85}
O'Keefe, R.~A. \BBOP1985\BBCP.
\newblock \BBOQ Towards an algebra for constructing logic programs\BBCQ\
\newblock In {\Bem SLP 1985}, \BPGS\ 152--160.

\bibitem[\protect\BCAY{Pelov}{Pelov}{2004}]{Pelov04}
Pelov, N. \BBOP2004\BBCP.
\newblock {\Bem {Semantics of Logic Programs with Aggregates}}.
\newblock Ph.D.\ thesis, Katholieke Universiteit Leuven, Leuven.

\bibitem[\protect\BCAY{Pelov, Denecker,\ \BBA\ Bruynooghe}{Pelov
  et~al.}{2007}]{Pelov07}
Pelov, N., Denecker, M., \BBA\ Bruynooghe, M. \BBOP2007\BBCP.
\newblock \BBOQ {Well-founded and Stable Semantics of Logic Programs with
  Aggregates}\BBCQ\
\newblock {\Bem Theory and Practice of Logic Programming}, {\Bem 7\/}(3),
  301--353.

\bibitem[\protect\BCAY{Simons, Niemel{\"a},\ \BBA\ Soininen}{Simons
  et~al.}{2002}]{Simons02}
Simons, P., Niemel{\"a}, I., \BBA\ Soininen, T. \BBOP2002\BBCP.
\newblock \BBOQ {Extending and implementing the stable model semantics}\BBCQ\
\newblock {\Bem Artificial Intelligence}, {\Bem 138}, 181--234.

\bibitem[\protect\BCAY{Truszczy{\'n}ski}{Truszczy{\'n}ski}{2006}]{Truszczynski06}
Truszczy{\'n}ski, M. \BBOP2006\BBCP.
\newblock \BBOQ Strong and uniform equivalence of nonmonotonic theories --- an
  algebraic approach\BBCQ\
\newblock {\Bem Annals of Mathematics and Artificial Intelligence}, {\Bem
  48\/}(3-4), 245--265.

\bibitem[\protect\BCAY{van Emden\ \BBA\ Kowalski}{van Emden\ \BBA\
  Kowalski}{1976}]{vanEmden76}
van Emden, M.~H.\BBACOMMA\  \BBA\ Kowalski, R. \BBOP1976\BBCP.
\newblock \BBOQ The semantics of predicate logic as a programming
  language\BBCQ\
\newblock {\Bem Journal of the ACM}, {\Bem 23\/}(4), 733--742.

\end{thebibliography}
\end{document}